\documentclass[nonblindrev,copyedit]{informs1} 

\usepackage{amsmath, natbib, bm, mathrsfs, url, multirow, enumerate, xcolor}
\usepackage[normalem]{ulem}
\usepackage{natbib}
\bibpunct[, ]{(}{)}{,}{a}{}{,}%
\def\bibfont{\small}%

\newcommand*{\set}[1]{\left\{ #1 \right\}}

\usepackage[font={small}]{subcaption}
\usepackage{graphicx}
\usepackage{array}
\usepackage{booktabs}
\usepackage{multirow}
\usepackage{algorithm} 
\usepackage{algpseudocode} 

\usepackage{hyperref}
\usepackage{xcolor}
\definecolor{hypercolor}{RGB}{25,25,112} 
\hypersetup{
    colorlinks=true,
    linkcolor=hypercolor,
    filecolor=hypercolor,   
    citecolor=hypercolor,   
    urlcolor=black,
}

\newcommand{\non}{\nonumber}

\newcommand{\fignote}[1]{
\begin{flushleft}
    {\footnotesize {\it Note:} #1}
\end{flushleft}
}
\TheoremsNumberedThrough     

\MANUSCRIPTNO{} 

\begin{document}

\RUNAUTHOR{Kang et al.}

\RUNTITLE{An Inception Framework for Image-Based SPC}

\TITLE{Deep Vision in Smart Manufacturing: MODERN Framework for Intelligent Quality Monitoring and Diagnosis}
%

\ARTICLEAUTHORS{
	\AUTHOR{Yicheng Kang}
	\AFF{Farmer School of Business, Miami University, OH. USA, \EMAIL{kangy10@miamioh.edu}}
	\AUTHOR{Yuling Jiao}
	\AFF{School of Mathematics \& Statistics, Wuhan University, Hubei, China, \EMAIL{yulingjiaomath@whu.edu.cn}} 
	\AUTHOR{Xin Geng}
	\AFF{Miami Herbert Business School, University of Miami, FL. USA, \EMAIL{xgeng@bus.miami.edu}} 
	\AUTHOR{Mahesh Nagarajan}
	\AFF{Sauder School of Business, The University of British Columbia, BC. Canada, \EMAIL{mahesh.nagarajan@sauder.ubc.ca}} 
} 

\ABSTRACT{
Smart manufacturing processes are often installed with a large number of sensors, imaging devices and computers, which not only enable instant communication across various modules of a production system but also aid in intelligent manufacturing management. 
In this paper, we introduce ``MODERN'', a deep learning framework for quality monitoring and fault isolation, which integrates these enhanced capabilities into the practice of industrial quality control. Using the architecture of an inception residual neural network, we develop a control chart that monitors the likelihood of a product containing defects. We also propose a faulty region estimator that identifies the defective area using transfer learning. To extend our framework to cases where there are not sufficient training data, we suggest a transfer monitoring technique that requires only a small sample size and a hypothesis testing approach for quantitatively assessing the applicability of our method. Theoretically, we establish the minimax optimal convergence rate for both our defect likelihood estimation and fault diagnosis. Our results lead to a seemingly counter-intuitive managerial implication - it may not always be in a manufacturer's best interests to keep upgrading its monitoring equipment regardless of the cost. Empirically, we demonstrate the superior performance of our method in comparison with a state-of-the-art approach using both simulated experiments and real data. 
}%


\maketitle

\section{Introduction}
\label{sec:1}

\subsection{Background and Motivation}\label{intro_mot}

Over (nearly) 100 years since Statistical Process Control (SPC) was introduced by Shewhart (1931), SPC methods have always been widely used across industries, especially manufacturing, to monitor production process and control product quality. 
%
Among others, control charts as the main technique of SPC aim to quickly detect any assignable causes of variation so that investigation and correction can be done in time. 
If used effectively, SPC methods play an important role in decreasing the production variation and reducing the overall cost of quality. As reported by the American Society for Quality\footnote{Source: \url{https://asq.org/quality-resources/cost-of-quality}.}, the true quality-related costs could reach to 15-20\% of sales revenue, some going as high as 40\% of the operating expenses in manufacturing firms. Hence, setting up effective control charts can help firms substantially reduce cost and make direct contribution to profit. 
Aside from the practical impact, there are many theoretical studies of SPC control charts in the management science literature. Starting from the classic Shewhart's $x$-bar chart and $p$-chart, researchers have developed many control charts with more sophisticated design for quality characteristics of either continuous variables or attributes (see, e.g., \citealt{page1961cumulative, montgomery1972economic, lowry1992multivariate, hawkins2003changepoint}). The traditional design of control charts involves selecting the sample size, sample frequency, and control limits to satisfy certain statistical and/or economic requirements. Examples include the sensitivity to false alarm leading to requirement on in-control (IC) \emph{average run length} (ARL) and the timeliness of delivery of out-of-control (OC) signals. 
Naturally, given the costs of sampling, investigating, and correcting, traditional designs often seek to minimize the average long-run total cost (see, for example, \citealt{jensen1984monitoring, tagaras1988economic, calabrese1995bayesian, porteus1997opportunities, wang2015multistate, wang2016minimizing, montgomery2019introduction}). 
%


It is noteworthy that, even though control charts are important tools for quality monitoring, the research on this topic has quieted down in the past decade. This is because the design and analysis of traditional control charts have been well understood. However, in recent years the landscape of quality monitoring has been drastically changed by rapid advancements of emerging technologies, necessitating a fresh look at the SPC techniques in modern manufacturing. The high degree of automation in modern manufacturing comes with quality monitoring systems which use powerful technologies including computer vision, robotics, and artificial intelligence --  such digitization throughout the entire manufacturing process is referred to as \emph{smart manufacturing}, an example of Industry 4.0 \citep{olsen2020industry}. 
In contrast, in traditional manufacturing settings, human workers and analog machines are tasked with picking non-conforming products. These impose severe restrictions on their ability to keep up with high production volumes. Thus, only a fraction of products are sampled. Because of this, the design of traditional control charts focuses on sample size and inspection frequency. In smart factories, these trade-offs are largely absent. Sensors and cameras are able to unceasingly ``watch'' the products and capture information about product specifics such as dimension, surface defects and spatial patterns in real time.  

Since the essential objective of SPC, i.e., to monitor the potential process shifts and inform decision-making in real time, remains the same even in Industry 4.0, the SPC control charts need to evolve and adapt to the technologies. In fact, 
the investment in the advanced quality inspection technology has set a even higher requirement for the {\it responsiveness} and {\it accuracy} of the control chart. Specifically, to harness the power of the computer vision for smart manufacturing quality control, how should manufacturers recognize and classify high volume of images captured by cameras, and how should the image data be processed and analyzed to inform the control decisions in a timely fashion and offer post-signal diagnosis with high accuracy? This new challenge motivates us to exploit advanced artificial intelligence techniques, coupled with appropriate statistical methods, to build a general framework for image-based quality monitoring and fault diagnosis.

\subsection{Literature Review on Image-Based Quality Control}\label{intro_review}

Machine learning, especially deep learning, has revolutionized image analysis. While classical analysis methods could be easily rendered infeasible because of the high dimensional input (pixels in millions) of an image, methods based on deep learning can effectively deal with the dimensionality issue and efficiently extract image features for analysis. In particular, deep \emph{convolutional neural networks} (CNNs) \citep{freshlook}, the most commonly adopted deep learning basis, use an architecture of many convolutional and pooling layers to gradually reduce the dimensionality of the input data while keeping the most relevant features preserved. Specifically, convolutional layers are mathematical filters that preserve the translation-invariant features of images. They are usually followed by pooling layers which detect the locations and strengths of those features. The combined use of convolutional and pooling layers vastly reduces the number of parameters required by fully connected neural networks. 
Because of their effectiveness in image analysis, CNNs are widely used for image-related management in operations, marketing, and information systems \citep[see, e.g.,][]{davidms, liumk, kevin}. For image-based quality management, methods using CNNs have been studied in several industrial settings, including surface grading in ceramic tile production (\citealt{lopez2008performance}), anomaly detection in rolling processes in the steel industry (\citealt{feng2018difference}), monitoring nonconforming patterns in textile manufacturing (\citealt{bui2018monitoring}),  surface uniformity inspection in the liquid crystal display industry (\citealt{jiang2005liquid}), and defect detection in the welding industry (\citealt{zhang2019weld, park2019convolutional}). We refer the readers to \citet{survey2} and \citet{survey1} for surveys on the CNN applications to industrial inspection. 

In all of the above papers, images of products are collected, and the quality characteristics (e.g., product geometry or surface finish) extracted from these images are used to determine whether the manufactured products meet the quality standards. Therefore, these papers only deal with one aspect of quality control, namely, identifying defective \emph{products} via inspection. By contrast, we focus on the charting methods that detect persistent quality shifts in the \emph{process} via monitoring. Given the on-going nature of a typical manufacturing process, using control charts to manage the potential quality abnormality reduces the possible delay of out-of-control signal delivery, lowering the quality-related costs. Therefore, departing from the above line of works, our main research output is an image-based SPC control chart rather than a product quality inspection technique. 

The underlying logic for an image-based control chart to work is to check the conformity of a manufactured product by comparing its image with the pre-determined image of a good-quality product. So far, two types of monitoring methods have been proposed in the literature. However, they all impose restrictive assumptions on the captured images, limiting their use in many practical settings. Here we aim to propose a flexible method that is able to handle a large variety of images. To illustrate the limitations of the existing methods and the capability of our approach, we provide three types of images in different contexts; see Figure \ref{fig:1} below. 
\begin{figure}[htp]
  \centering
  \caption{An illustration for the insufficiency of the existing image-based quality monitoring.}\label{fig:1}
  \begin{tabular}{ccc}
    \includegraphics[width = 0.3\textwidth]{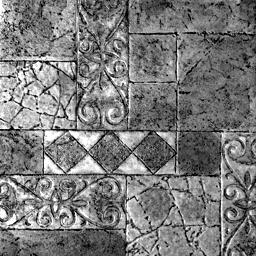} & \includegraphics[width = 0.31\textwidth]{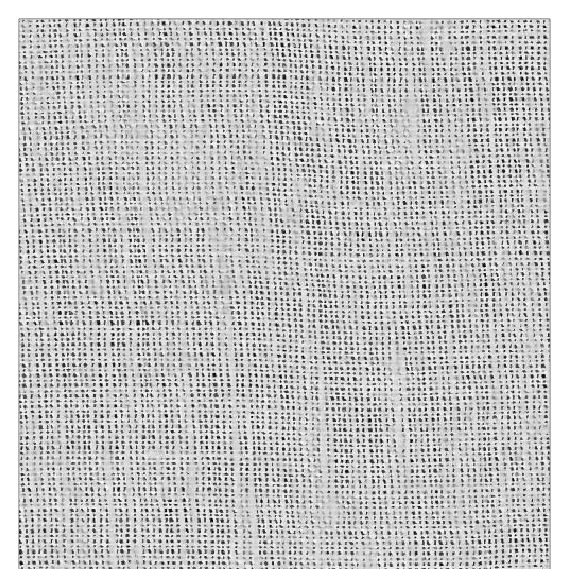} & \includegraphics[width = 0.302\textwidth]{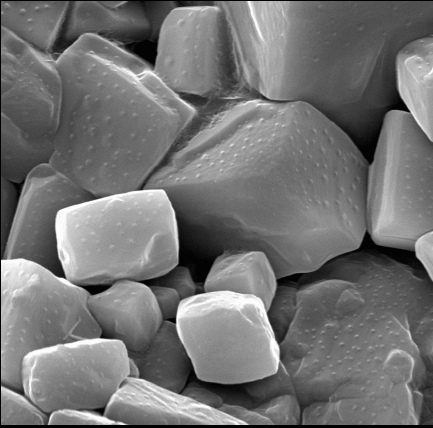} \\
    (a) A tile surface & (b) A textile fabric sample & (c) Microscopic structures of salt
  \end{tabular}
  \vspace{6pt}
  \begin{flushleft}
   {\footnotesize   {\it Note:} Image (a) has a deterministic pattern to conform to. Image (b) satisfies Markov Field conditions and its conformity can be defined in a stochastic sense. Image (c) does not have either property and no existing process monitoring method can treat this type of images well.  }
  \end{flushleft}
\end{figure}
First, Figure \ref{fig:1}(a) is an image of a tile with no surface defects. This image encompasses characteristics of good-quality products and can be used as a gold standard with distinct deterministic features. As a result, when monitoring this tile production process, images with significant deviation from the gold standard would be deemed out-of-control (OC). Most existing image monitoring methods in the literature are designed under the assumption that there is a gold standard image \citep[see, e.g.,][]{jin1999feature, lin2007computer, lin2007automated, megahed2012spatiotemporal, koosha2017statistical, feng2018difference, kang2022statistical}. 
Second, Figure \ref{fig:1}(b) displays an image of a textile fabric sample. Clearly, in this setting, there does not exist a gold standard as two good-quality products may look completely different in the sense that there is no pixel-to-pixel match between them. In fact, the in-control (IC) textile weave patterns exhibit stochastic behavior and are not deterministically repeated. Images of this type are considered in \citet{levina2006texture} and \citet{bui2018monitoring}. Their method relies crucially on the assumption that the joint distribution of pixel intensities (i.e., the brightness measure at each pixel location) must follow a Markov Field (MF) model. Specifically, the marginal distributions of individual pixels are the same regardless of the pixel location (i.e., the stationarity condition), and the dependence of a pixel's intensity on the other pixels is only through several neighboring pixels (i.e., the locality condition). 
%
%
These conditions seem to hold for Figure \ref{fig:1}(b), but it may not be the case in other applications. 
Lastly, Figure \ref{fig:1}(c) shows an image of the microscopic structures of margarita salt. In contrast to the previous two images, this image neither has a gold standard, as the minuscule cubes of salt are randomly placed together, nor does it satisfy the stationarity condition, as the small cubes appear brighter than the larger ones. Thus, the existing methods will all fail at the process monitoring and quality control for products with this type of images. To the best of our knowledge, we are the first to build a general CNN-based framework for image-based SPC control chart that has the capability to handle a variety of images without making restrictive assumptions (e.g., gold standard, locality, or stationarity).


\subsection{Results Overview and Contributions}\label{intro_con}


Our work has three sets of contributions. First, we provide a ready-made and implementable quality control technique that can be adapted in smart manufacturing. Second, we show that our method has attractive theoretical properties. Third, we are able to provide managerial insights on interpreting the results as well as guidance on decisions that may improve monitoring and control. We detail our contributions as follows.


\emph{Practical Contribution.} 
The foremost contribution of our research is to propose MODERN -- \underline{Mo}nitoring with \underline{De}ep \underline{R}esidual \underline{N}etwork -- a deep learning framework for image-based SPC quality control. The proposed framework accomplishes the following two goals that ensure successful manufacturing quality control: (i) \textit{quality monitoring}, i.e., quick delivery of a signal if the production process starts working abnormally, and (ii) \textit{post-signal diagnosis}, i.e., accurate isolation of the faulty region of the defective products after an OC signal is given by the control chart. Using the architecture of a deep neural network with residual connections, we are able to describe a wide range of deterministic and stochastic IC behaviors, filling the gap in the literature of image-based control charts. Based on the general framework and the neural network model, we provide an exponentially weighted moving average (EWMA) control chart (MODERN-Chart) to monitor possible process shifts and signal abnormality in real time. Our proposed monitoring technique has a desirable re-starting mechanism that accumulates the evidence when the observed images suggest a process shift and resets to the initial state when the data suggest otherwise. The implication is that our control chart is able to detect small but persistent shifts for the manufacturing process, which may be overlooked if only inspection methods are used for individual products. Moreover, the proposed chart is easy to implement and thus can be adopted in practice without much difficulty. In addition, by modifying the initialization and the final layer of the neural network, we develop a faulty region estimator to facilitate the post-signal diagnosis (MODERN-Diagnosis) by locating the defective area in the OC images. We show that the faulty region estimator can be calibrated on a relatively small dataset of defective images. {\color{black}It is worth noting that the MODERN framework is not restrictive in terms of the specific choice of the neural network. Since its implementation remains the same if a different neural network is used, it can evolve at the same pace as deep learning technology. In this sense, our research provides a road map for dynamically integrating deep learning technology into the practice of quality control. }



We also establish an important generalization that ensures a convenient implementation of MODERN framework in a wide range of manufacturing processes. Since different production lines have different sets of IC/OC images, a common but major challenge to apply deep learning to a specific setting is to quickly train the neural network from scratch to a sufficient level, especially when there is a relatively small set of training images. To overcome this obstacle, we propose a novel \emph{transfer monitoring} technique that consists of two components, namely, an applicability test and an augmented bootstrap algorithm for control limit determination. Given a specific manufacturing setting, we first conduct an applicability test assessing whether our pre-trained MODERN-Network is capable of distinguishing the IC and OC images at hand. The applicability test also safeguards against misuses of the framework, avoiding the cost of quality mismanagement. Then, we propose the augmented bootstrap algorithm that adapts the control limit of our MODERN-Chart to the specific setting. Notably, compared to retraining the neural network, the sample size required for our augmented bootstrap algorithm to recompute the control limit is significantly smaller. {\color{black}Hence, in the context of image-based quality control, our transfer monitoring solves the notorious problem of the formidable sample size requirement of deep learning technology. Furthermore, we provide the mathematical basis for this sample size reduction property by proving that the convergence rate by transfer monitoring is much faster than the convergence rate achieved by retraining the neural network.}


\emph{Theoretical Contribution.} 
For both the defect likelihood estimation and post-signal fault diagnosis based on the proposed framework, we establish the \emph{minimax asymptotic optimality}. Specifically, we first prove that the error of our quality control signaling and defect region isolation will approach to zero as the sample size grows. We then show, using novel techniques, that our neural network estimation achieves the minimax optimal rate of convergence. That is, no other methods can attain a faster rate, suggesting that our MODERN framework makes very efficient use of the training sample. 
The main theoretical challenge is that the underlying deep CNNs have a complex nonlinear structure, which prevents us from applying conventional approximation tools and using basis series such as polynomials, splines, or wavelets \citep[see, e.g.,][]{kang2022statistical}. 
As a result of this difficulty, the associated optimization problem cannot be reduced to solving a finite set of stationary equations. {\color{black}While the desired approximation result has been established for fully connected neural networks in recent developments of deep learning theory (e.g., \citealt{shen2019deep}), no such result exists for CNNs. To address this gap, we adapt existing techniques to approximate the loss functional in an infinite-dimensional space by constructing and bounding it by locally quadratic functionals. Ultimately, we show that the same convergence rate holds for CNNs as well. }


{\color{black}We also address an important matter in image-based quality control: the impact of mislabeling. Industrial images can get labeled erroneously due to negligence of the quality personnel or poor resolution of the image. Thus, in Industry 4.0, it is crucial for quality managers to understand the effect of mislabeling when quality control is integrated with deep learning technology. Our research shows that the estimation error of the defect likelihood can be decomposed into two terms: the CNN convergence rate and the image mislabeling rate. This allows quality managers to set a bound on the mislabeling rate relative to the total sample size, with the understanding that the deep learning-based quality control becomes less effective if the bound is exceeded.} 
%

\emph{Managerial Contribution.} 
As an extension to the previous two contributions, we also offer managerial insights for manufacturers who are transitioning to Industry 4.0. First, compared to image-based quality control using MODERN framework, the traditional manual inspection approach inevitably suffers from incompleteness and subjectivity: Human operators are easily subject to errors and fatigue, and they may have different training and visual perceptions. Overcoming these drawbacks, MODERN framework for smart manufacturing quality control achieves higher efficiency (machines do not need rest) and greater repeatability (human subjectivity is eliminated). 
Second, our convergence results have useful implications on the equipment/device investment. Despite the benefit of accumulating additional training images (the estimation error would diminish), the convergence rate is found to be decreasing in the image resolution, suggesting that it may not always be in a manufacturer's best interests to keep upgrading its monitoring equipment for higher image resolution. Instead, {\color{black}the manufacturer should consider image quality together with labeling quality. 
In other words, if mislabeling does not occur or occurs at a tolerable level (i.e., within the aforementioned asymptotic bound), then it would be beneficial for the manufacturer to not upgrade or even downgrade their imaging equipment. On the other hand, if the mislabeling rate relative to the total sample size exceeds the asymptotic bound, our deep learning-based quality control becomes less effective. In such cases, higher-resolution images are necessary to allow precise labeling and thus investments in equipment upgrades are warranted.} Third, our MODERN framework offers \textit{interpretable} diagnosis to quality managers. Following an OC signal given by MODERN-Chart, MODERN-Diagnosis identifies the specific faulty part of the defective product that triggers the OC signal, providing useful information about the root cause.

%

\vspace{8pt}


Finally, we conduct a series of extensive numerical studies to showcase the performance of the proposed MODERN framework. 
First, using DAGM, a commonly used benchmark dataset for industrial image analysis, we evaluate the numerical performance of our method. MODERN-Chart is observed to give a signal at the exact change point (where process just starts to shift) most of the time, a satisfactory performance for quality monitoring. The numerical assessment also reports that our faulty region estimator does well in locating the defect, even though the estimator is trained on a small set of OC images. 
Second, we compare our method with a state-of-the-art approach in simulated experiments. Since the images in our setting typically do not have a gold standard, the comparison is done with the Markov Field (MF) approach suggested by \citet{bui2018monitoring}. 
Using synthetic images that satisfy the Markov Field conditions, we show that our method performs better than the MF approach in terms of timely signal delivery and accurate faulty area isolation. 
%
Third, we demonstrate the proposed method on {\color{black} two sets of real product images, one from a commutation system manufacturer and the other from a machine software supplier.} Our method is shown to deliver the OC signal quickly and flag the faulty area. We also apply the MF approach to the same sets of images and find that it may send a false alarm before the process starts to shift. 
%



\subsection{Organization of the Paper}\label{intro_org}

The remainder of this paper is organized as follows. Section \ref{sec:2} describes our model for the image data analysis based on a deep CNN of inception architecture with residual connections. 
Section \ref{sec:3} introduces the proposed MODERN framework. Section \ref{sec:4} discusses theoretical results and their implications. Section \ref{sec:5} presents a benchmark dataset of industrial images and evaluates the proposed method using numerical studies. Section \ref{sec:6} summarizes the work and suggests possible extensions beyond industrial manufacturing. Technical details are relegated to the appendix.

\section{Model Descriptions} \label{sec:2}

Extracting information from images is challenging because image data are unstructured and contain complex spatial patterns. Although composed of pixels, an image cannot simply be treated as a collection of individual pixels; otherwise, the spatial information, which captures the main characterizations of the image, would be lost. Therefore, it is necessary to create a set of meaningful features that best represent the image. Manually defining such features is time-consuming and often domain-specific, making it impractical and non-generalizable. The recent development of deep learning offers a much more efficient tool for image analysis. In particular, deep CNNs automate the step of feature extraction by learning the representative features from data. Since its invention, deep CNNs have seen a number of successes in difficult tasks such as image classification \citep{krizhevsky2017imagenet}, image segmentation \citep{long2015fully}, human pose estimation \citep{toshev2014deeppose} and object tracking \citep{wang2013learning}. Refer to \citet{stevens2020deep} for a more detailed discussion. 
In the following, we describe our image analysis model, which is built upon the deep convolutional neural networks. First, we introduce the transformation layers as the basic building blocks for the network. Then, in order to deal with a large amount of layers, we incorporate an efficient deep CNN design into our model.

\subsection{Basic Building Blocks}\label{sec:2.1}

A CNN involves various types of transformation layers. The output of each layer serves as the input into the next one. It is through these layers that a CNN learns the key attributes of an image. Typically, the first several layers extract basic image features, deeper layers detect complete objects, and the final layer returns outcome variables with meaningful interpretations. Same as most of CNNs seen in literature (e.g., \citealt{krizhevsky2012imagenet, lecun2015deep}), our model utilizes four types of transformation layers: convolutional, activation, pooling and fully connected layers. We briefly describe each type next.


\subsubsection{Convolutional Layer.} \label{sec:2.1.1}
Consider an input image represented by a pixel matrix, $X_{\texttt{in}}$, and a corresponding \textit{kernel}, which is a weight matrix, $\{K(i, j): -k \leq i, j \leq k\}$, of size $2k + 1$\footnote{Odd-sized kernels are often used in convolutional layers for analytical convenience.}. Then, the output data of a convolutional layer, denoted by $X_{\texttt{out}}$, is element-wise computed by sliding the kernel on all input locations and performing the weighted sum. Specifically, the pixel at location $(m, n)$ of the transformation output is defined as the scalar product of the kernel $K$ with the neighborhood of input pixel $X_{\texttt{in}}(m, n)$: 
\begin{align*}
  X_{\texttt{out}}(m, n) = \sum_{i = -k}^{k}\sum_{i = -k}^{k} X_{\texttt{in}}(m + i, n + j) \cdot K(i, j) + b,
\end{align*}
where $b$ is the bias parameter associated with the kernel. Note that the size of the output image can vary according to different designs of the layer. For example, the number of pixels the kernel slides, called \textit{stride}, can affect the output size; i.e., the larger the stride, the smaller the output image. Another factor is how the convolution is defined on the boundary of the input image, referred to as \emph{padding}. There are two approaches for padding. \textit{Same padding} adds pixels of value zero around the border so that the output image has the same size as the input image, whereas \textit{valid padding} means sliding the convolution kernel within the input image, resulting in an output image smaller than the input. Therefore, the dimension of $X_{\texttt{out}}$ is determined by the dimension of $X_{\texttt{in}}$, the kernel size, the stride, and the padding approach. Moreover, there can be several kernels in a convolutional layer, each called a \emph{filter}. The number of filters defines the width of the layer, which is also the number of output images of that layer. The hyperparameters of the convolutional layer, i.e., stride, padding, kernel size, and filter quantity are all specified by the CNN's architecture (or user). The learnable parameters, including the kernel elements and biases, are estimated from data.

The key rationale behind convolutions is two-fold. First, it is a localized operation in the sense that it only examines how nearby pixels are arranged, disregarding pixels from far away. This is a reasonable design because patterns of objects in images are indeed localized. Second, the convolution transformation is translation invariant, because the same kernel is applied across the input image. This ensures that the location of an object does not matter if that object is to be detected in an image. For instance, if we want to recognize an airplane, we are not concerned with whether the airplane is in the upper left or lower right corner of the image. This parameter-sharing design facilitates a significant reduction of the number of parameters in the layer.

\subsubsection{Activation Layer.}\label{sec:2.1.2}
An activation layer applies an element-wise nonlinear activation function to input from the previous layer. Thus, the dimensions of the output and input are the same. While convolutional layers, no matter how many times they are applied to input data, only result in an affine/linear functional form, the nonlinearity introduced by activation layers allows the CNNs to approximate more complex functional relationships. There are a few popular choices for activation functions. We use one of the best-performing activation functions, namely, the ReLU (rectified linear unit) activation in our model: 
\begin{align*}
  \text{ReLU}(x) = \max\set{0, x}.
\end{align*}
Therefore, the ReLU activation layer has a threshold property: pixels without the feature defined by the corresponding kernel are suppressed to zero values in the subsequent layers. Due to this property, ReLU activation layers can detect signal strengths and signal locations in our model. 

\subsubsection{Pooling Layer.}
\label{sec:2.1.3}
A pooling layer is often inserted between convolutional layers to transform the features of the input data in a collective manner. Two commonly seen pooling operations, namely,  max pooling and average pooling, are used in our model. Max [average] pooling takes the maximum [average] over each neighborhood of the input image as the new pixel of the output image. Similar to convolutional layers, it has hyperparameters such as the stride and the size of the pooled neighborhood that determine the output dimension. Moreover, if the previous layer has multiple filters, pooling operates independently on every filter so it has the same width. 
%
%
Intuitively, the output images from a convolution layer followed by a ReLU activation tend to have a high magnitude at places where features associated with some kernels are detected. By keeping the maximum or the average in the neighborhood as the down-sized output, the pooling layer ensures that the detected features are preserved. 

A convolutional layer, activation layer and pooling layer together form important building blocks for constructing CNNs. Most CNNs use the three-layer building blocks in a repeated fashion, with varying dimensions. Typically, as the network gets deeper, the output image progressively shrinks in size while the number of filters increases.

\subsubsection{Fully Connected Layer.}
\label{sec:2.1.4}
Fully connected layers are usually at the end of a CNN, where the final output is about to be produced. Output images from the previous layer are first flattened as a vector, which is then fed into a fully connected layer that operates as follows. Let $Z_{\texttt{in}}$ be an input vector of length $l_{\texttt{in}}$ and $W_{\texttt{FC}}$ be the weight matrix of size $l_{\texttt{out}} \times l_{\texttt{in}}$ associated with the fully connected layer. Then, the output vector $Z_{\texttt{out}}$ is computed as
\begin{align*}
  Z_{\texttt{out}} = W_{\texttt{FC}} \; Z_{\texttt{in}} + b_{\texttt{FC}},
\end{align*}
where $b_{\texttt{FC}}$ is the bias parameter of length $l_{\texttt{out}}$. Hence, the fully connected layer includes a total of $l_{\texttt{out}} \times l_{\texttt{in}} + l_{\texttt{out}}$ learnable parameters. Unlike convolutional layers where each output pixel is connected only to a local region of the input, units in a fully connected layer are connected to the entire subsequent layer. As such, the flattened vector is often regarded as derived image features. The main purpose of fully connected layers is to map these features to target variables with contextual interpretations. 
For quality monitoring where the objective is to signal whether the process has shifted, the target variable is the probability of the process being OC. 
Given that the output unit from a fully connected layer can take any real value, we relate it to the OC probability via the \textit{logistic} function, defined as
\begin{align*}
  p(z) = \frac{e^{z}}{1 + e^z},
\end{align*}
where $z$ is the output from a fully connected layer.  In the post-signal diagnosis model, however, the last layer output $\bm{z}$ is a five-dimension variable that represents the location of the faulty region, and there is no need to apply the logistic function in that case. See Section \ref{sec:3.3} for details.

%


\subsection{Inception Architecture with Residual Connections}
\label{sec:2.2}

Applying multiple transformation layers of different types gives a CNN model. If the number of convolutional layers is relatively small, the CNN model is too shallow to perform satisfactorily. \emph{Deep} CNNs refer to those with many (usually $>100$) convolutional layers, and the resulting model capability is increased. Indeed, with depth, the CNN model is able to approximate much more complex functional relationships and even capture the hierarchical information when we need to understand the context. For instance, a shallow network could only identify a person's shape in a photo, but a deep network could identify the person's face and the mouth in the face. 

The inception architecture, initially proposed by \citet{szegedy2015going}, is an efficient deep CNN design for computer vision. At the ImageNet Large-Scale Visual Recognition Challenge in 2014, the first incarnation of inception-style CNN, ``GoogLeNet'', outperformed ``AlexNet'', which was the state-of-the-art deep CNN at that time \citep{krizhevsky2012imagenet}. Since then, the inception design has been improved and further refined by \citet{szegedy2016rethinking} and \citet{szegedy2017inception}. 
%
%
The inception design aligns with the intuition that visual information should be processed at various scales and then aggregated so that the next stage can abstract features from different scales simultaneously. Based on this rationale, the key feature of an inception architecture is \textit{filter concatenation}, which concatenates the output images from different filters to become what is called an inception module; and the inception-style CNNs involve repeated use of the inception modules, which are stacked upon each other in the network. Figure \ref{fig:incep_res}(a) provides a graphical illustration of filter concatenation using a generic inception module as an example. 

%

%
\begin{figure}[ht!]
\centering
\caption{Graphical illustration of filter concatenation and residual connection in a generic inception module.} \label{fig:incep_res}
\begin{subfigure}[t]{.45\linewidth}
\centering
			 \includegraphics[width = \linewidth]{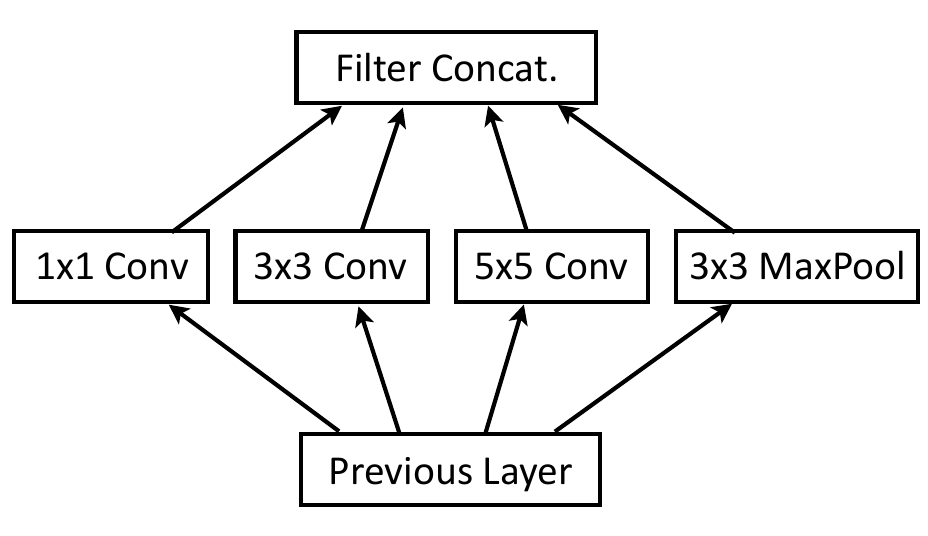}
             \caption{An example of filter concatenation.}  
\end{subfigure}
~ 
\begin{subfigure}[t]{.45\linewidth}
\centering
			\includegraphics[width=\linewidth]{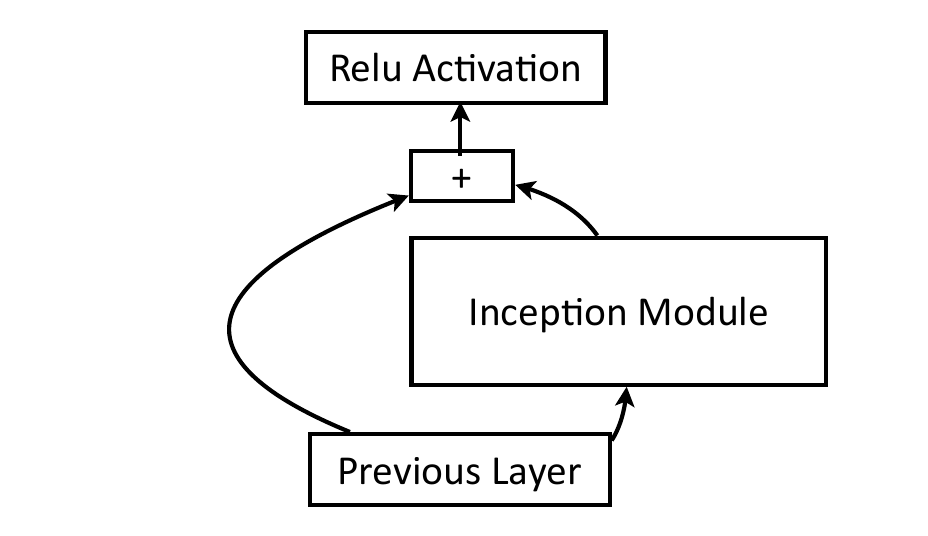} %
            \caption{An illustration of residual connection.}  
\end{subfigure}

\vspace{8pt}
\fignote{Several filters of a transformation layer are concatenated into an inception module. Then, that inception module and the previous layer are fed together into the next activation layer.}
\end{figure}

A major challenge with deep CNNs is that the training becomes harder as the depth increases. Indeed, features carried by an early layer could easily vanish after undergoing a long chain of transformation, leading to ineffective training of that layer. The \textit{residual connection} design suggested by \citet{he2016deep} addresses this issue. The main idea is to skip layers by adding the input to the output before the next activation layer, forcing the features learned from the shallower layers to be carried over to the deeper layers; see Figure \ref{fig:incep_res}(b) for an illustration and note that the residual connection is represented by the ``+'' box. Hence, adding residual connections to the inception architecture accelerates the training of the networks while maintaining the performance. This idea was incorporated by \citet{szegedy2017inception} in their latest version of inception design ``Inception ResNet'', which shows high computational efficiency and excellent approximation performance. Therefore, we will adopt the residual inception architecture for the neural network in our image data analysis model, which is further described next.

\subsection{The Network Architecture and Hyperparameters in Our Model}\label{sec2_ours}

Our model utilizes the basic building blocks of various transformation layers from Section \ref{sec:2.1} and incorporates the residual inception architecture from Section \ref{sec:2.2}. Being a deep CNN model with numerous layers where multiple inception modules are applied at different places, the model has a complex structure and many hyperparameters; In particular, our network has about 23.4 million learnable parameters. 
To present our inception residual CNN model, Figure \ref{fig:schema} provides a complete schematic view of the network architecture and hyperparameters. At the center of the figure, there is a (linear) top-level structure comprising six main modules of operations, namely, ``Stem'', ``InceptionRes-A'', ``Reduction-A'', ``InceptionRes-B'', ``Reduction-B'' and ``InceptionRes-C''. Note that ``InceptionRes-A'' module is repeated five times in the architecture, ``InceptionRes-B'' ten times, and ``InceptionRes-C'' five times. Moreover, the inception modules in Figure \ref{fig:schema}, e.g., filter concatenations and residual connections, operate in the way as described in Section \ref{sec:2.2}. 

The main modules are further depicted as indicated alongside the main structure with detailed information. As the main element of the network, convolutional layers are used many times. Each of them is represented by a box named ``Conv'', with all the relevant hyperparameters reported in the box. Specifically, kernel size is given (e.g., 3$\times$3) on top and ``(filter quantity, stride, padding approach)'' are given in the parenthesis. 
Furthermore, the box named ``MaxPool'' refers to a pooling layer with max operation and the hyperparameters are given as ``[stride, padding approach]'' in the box. Lastly, after the six main modules of operations, the ``Final Layers'' in the figure refer to an average pooling layer and a fully connected layer with the logistic function. To avoid overfitting, we set a dropout probability $0.6$ to the output data before applying the softmax function.


\begin{figure}[ht]
  \centering
  \caption{The complete schematic view of MODERN-Net architecture.}
  \label{fig:schema} 
    \includegraphics[width = \textwidth]{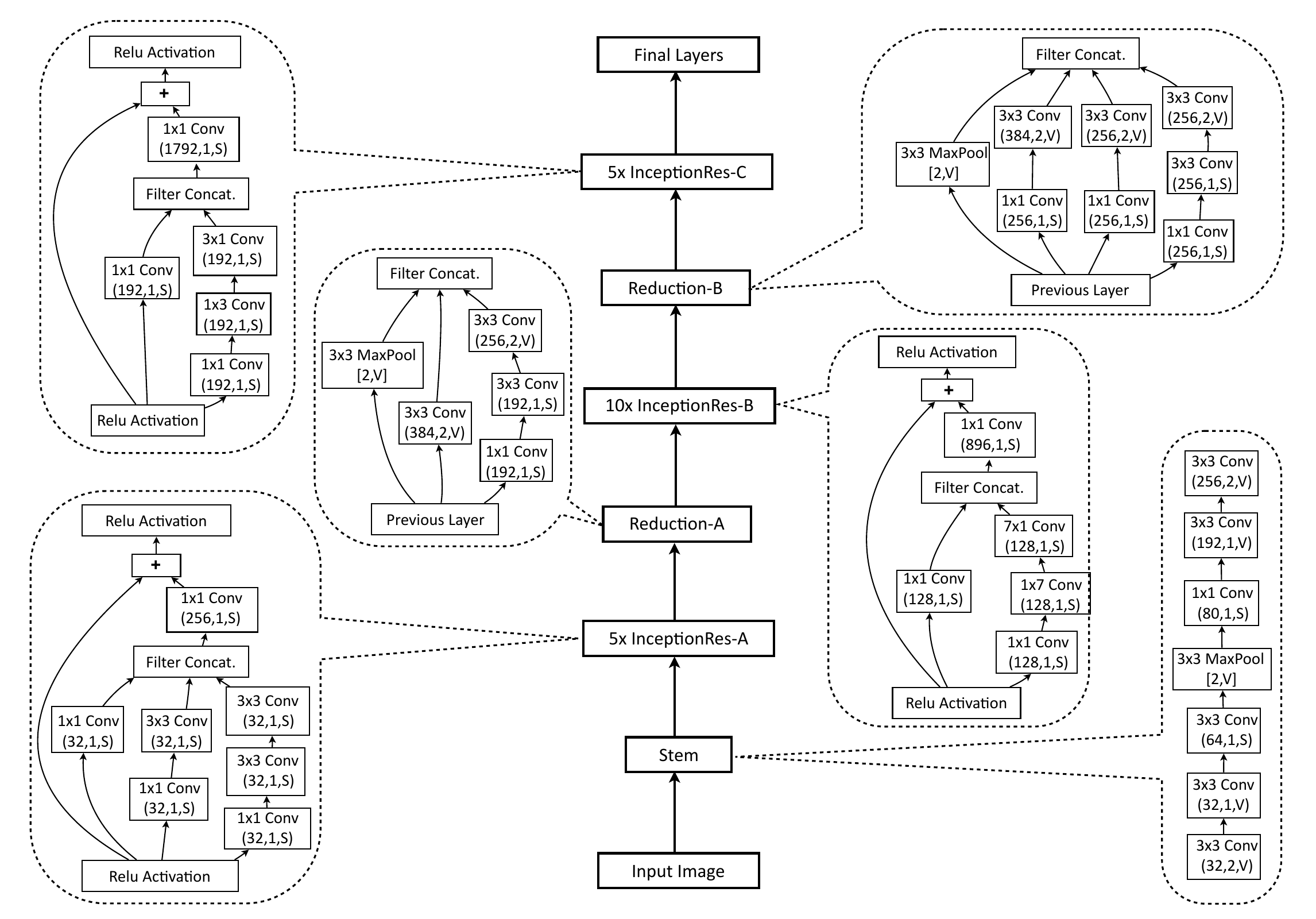} 
\end{figure}

\section{The Proposed MODERN Framework}\label{sec:3}


The underlying architecture of MODERN has been described in Figure \ref{fig:schema}. In this section, we specify how the network is trained and, more importantly, how the control chart and the faulty region estimator are constructed using the MODERN framework. We also show how to extend the applicability of our MODERN framework to different industrial settings.

\subsection{The Training of MODERN-Net}\label{sec:3.1}

We first discuss the core technique for training the CNN in our model, i.e., the stochastic gradient descent method with non-constant learning rate. Recall that our quality control uses the MODERN framework to accomplish two tasks, signaling OC events during the process monitoring and isolating the faulty regions on defective products. Although the outputs of these two tasks are different, they are based on a common set of image features. Hence, we focus on the first task, where the objective of the CNN is to estimate the probability of the input image containing defects. In Section \ref{sec:3.3}, we will modify this network to become the faulty region estimator. 

Let $\mathscr{F}$ represent our focal network, which is referred to as MODERN-Net hereafter. Furthermore, we use $\mathscr{F}(\bm{X}; \mathcal{V})$ to denote the outcome probability from applying MODERN-Net $\mathscr{F}$ to an input image $\bm{X}$ using the collection of learnable parameters $\mathcal{V}$. Then, given a set of training images, $\{(\bm{X}_i, y_i): 1 \leq i \leq N \}$, where the $i$-th image $\bm{X}_i$ is either normal ($y_i = 0$) or defective ($y_i = 1$), 
%
%
we train MODERN-Net by finding the parameters in $\mathcal{V}$ to minimize the cross-entropy loss 
\begin{align*}
 \frac{1}{N}\sum_{i=1}^N L\left( y_i, \mathscr{F}(\bm{X}_i; \mathcal{V}) \right) = \frac{1}{N}\sum_{i=1}^N -\Big\{ y_i\log\left( \mathscr{F}(\bm{X}_i; \mathcal{V})\right) + (1 - y_i) \log\left( 1 - \mathscr{F}(\bm{X}_i; \mathcal{V}) \right) \Big\}.
\end{align*}
One way to solve the above minimization problem is through gradient descent method, which is to iteratively update the decision variable $\mathcal{V}$ along the steepest descent until convergence. Specifically, the gradient-descent update is
\begin{align*}
  \mathcal{V} \leftarrow \mathcal{V} - \alpha \cdot \Delta\mathcal{V}, \quad \Delta\mathcal{V} = \frac{1}{N}\sum_{i=1}^N \frac{\partial L\left( y_i, \mathscr{F}(\bm{X}_i; \mathcal{V}) \right)}{\partial \mathcal{V}},
\end{align*}
where $\alpha \in (0, 1]$ is the \textit{learning rate}. 
We modify the above method in two aspects. First, instead of processing all pairs $\{(\bm{X}_i, y_i)\}$ before making a gradient step, we process randomly sampled smaller batches to decide the gradient-descent update. That is, the gradient step is given by 
\begin{align*}
  \Delta\mathcal{V} = \frac{1}{N_k}\sum_{\bm{X}_i \in \mathcal{B}_k} \frac{\partial L\left( y_i, \mathscr{F}(\bm{X}_i; \mathcal{V}) \right)}{\partial \mathcal{V}},
\end{align*}
where $\mathcal{B}_k$ denotes the $k$-th batch of size $N_k$. By estimating the gradient on a few samples at a time, we prevent the optimization process from getting stuck in local minima. We process the smaller random samples until all training pairs have been used in the gradient step calculations. This is called an \emph{epoch}. Epochs of updates are repeated until the cross-entropy loss converges. The training set is shuffled at each epoch, so that the gradient-descent update is accompanied by stochasticity of sampling. 
By introducing randomness to this process, we ensure that the sequence of gradients estimated over small batches is representative of the gradients computed across the full set.


The second modification to the gradient-descent update is that we consider variable, rather than constant, learning rate. With a constant learning rate $\alpha$, the algorithm may converge slowly if $\alpha$ is too small but could diverge if $\alpha$ is too large, limiting the performance of the network. 
%
%
Therefore, it is beneficial to vary the learning rate as the iteration progresses. Based on the work of \citet{smith2017cyclical}, we adopt the optimal learning rate policy with the triangular cyclical schedule: At the $j$-th iteration ($j = 1, ,2, \cdots$), the learning rate is
\begin{align*}
  \alpha_j =
  \begin{cases}
    \alpha_L + (\alpha_U - \alpha_L)\frac{j - l\cdot C}{C}, & l\cdot C < j \leq (l + 1)C \\
    \alpha_U - (\alpha_U - \alpha_L)\frac{j - (l + 1) \cdot C}{C}, &(l + 1)C < j \leq (l + 2)C
  \end{cases}, \quad l = 0, 2, 4, \cdots ,
\end{align*}
where $\alpha_L$ and $\alpha_U$ are the lower and upper bounds of the learning rate and $C$ is half the cycle, all of which are predetermined. Using the above variable learning rates in the cyclical schedule, near optimal learning rates are attained throughout the training, and thus the efficiency is improved.  


The above stochastic gradient descent method with variable learning rates can solve the cross-entropy loss minimization. The optimal solution leads to our trained MODERN-Net. Let $\widehat{\mathcal{V}}$ denote the estimated $\mathcal{V}$ after training. Define
\begin{align*}
  \widehat{p} = \mathscr{F}\left( \bm{X}; \widehat{\mathcal{V}} \right)
\end{align*}
to be the estimated probability of the product being defective given the input image $\bm{X}$. Therefore, $\widehat{p}$ is the key output of MODERN-Net for us to construct the control chart (referred to as MODERN-Chart) for quality monitoring. 

\subsection{Quality Monitoring by MODERN-Chart}
\label{sec:3.2}

Let $\{\bm{X}_t: t \geq 1 \}$ denote the images being collected from a manufacturing process. Since we are concerned about the process becoming OC, a signal should be delivered if $\widehat{p}$ is too high and the process should be stopped for root cause investigation. If $\widehat{p}$ remains low, then the process should keep running. Hence, we are only interested in detecting any upward shift in the sequence $\{\widehat{p}_t = \mathscr{F}(\bm{X}_t; \widehat{\mathcal{V}}): t \geq 1 \}$. To this end, we propose the following recursive charting statistics:
\begin{align}
  \label{eq:1}
  E_t = \max\left\{ 0, \lambda (\widehat{p}_t - \mu_{\texttt{IC}}) + (1 - \lambda) E_{t-1} \right\},
\end{align}
where $E_0 = 0$, $\lambda \in (0, 1)$ is a weighting parameter, and $\mu_{\texttt{IC}}$ is the mean of $\widehat{p}_t$ when the process is IC. Clearly, $E_t$ incorporates the accumulation of the possible OC evidence, which is represented by the upward deviation of $\widehat{p}_t$ from $\mu_{\texttt{IC}}$, over time. In particular, assuming that $\widehat{p}_t - \mu_{\texttt{IC}} \geq 0$ for all $t$, the accumulative OC evidence captured by $E_t$ is an exponentially weighted moving average (EWMA): 
\begin{align*}
  \sum_{s=1}^t \lambda(1 - \lambda)^{t - s}(\widehat{p}_s - \mu_{\texttt{IC}}) . 
\end{align*}
One can verify that the asymptotic variance of this EWMA sequence is $(\lambda \sigma^2_{\texttt{IC}})/(2 - \lambda)$, where $\sigma_{\texttt{IC}}$ is the standard deviation of $\widehat{p}_t$ when the process is IC. Therefore, a signal of upward shift should be sent if
\begin{align}
  \label{eq:2}
  E_t > \rho \sqrt{\frac{\lambda}{2 - \lambda} \sigma^2_{\texttt{IC}}}, 
\end{align}
where $\rho > 0$ is a parameter chosen to achieve a pre-specified IC average run length (ARL$_0$). 

Therefore, our MODERN-Chart is defined as the combination of the charting statistics \eqref{eq:1} and the signal criterion \eqref{eq:2}. It is worth noting that, compared with a conventional EWMA, MODERN-Chart has a desirable feature: The charting statistic $E_t$ is reset to $0$ every time when $\lambda (\widehat{p}_t - \mu_{\texttt{IC}}) + (1 - \lambda) E_{t-1} < 0$. Since there is little evidence suggesting OC in those cases, keeping the negative value in the charting statistics would only reduce the chart's sensitivity to upward shifts. Hence, such a re-starting mechanism helps MODERN-Chart to respond quickly to quality degrading. 

Four parameters need to be determined in order to use MODERN-Chart: $\lambda$,  $\mu_{\texttt{IC}}$, $\sigma^2_{\texttt{IC}}$ and $\rho$. 
The choice of $\lambda$ has been well studied in the literature, and we 
adopt one of the popular choices, $\lambda = 0.1$, throughout our discussion (see \citealt{prabhu1997designing} for a similar setup). 
To estimate $\mu_{\texttt{IC}}$ and $\sigma^2_{\texttt{IC}}$, suppose that the training dataset can be written as $\mathcal{C} = \mathcal{C}_0 \cup \mathcal{C}_1$, where $\mathcal{C}_0 = \{\bm{X}_k^{\texttt{IC}}: k = 1, \cdots, n_{\texttt{IC}} \}$ and $\mathcal{C}_1 = \{\bm{X}_k^{\texttt{OC}}: k = 1, \cdots, n_{\texttt{OC}} \}$ respectively denote the set of IC and OC images. Then, the estimated $\mu_{\texttt{IC}}$ and $\sigma^2_{\texttt{IC}}$ are given by
\begin{align*}
  \widehat{\mu}_{\texttt{IC}} = \frac{1}{n_{\texttt{IC}}}\sum_{k=1}^{n_{\texttt{IC}}}\mathscr{F}\left( \bm{X}_k^{\texttt{IC}}; \widehat{\mathcal{V}} \right), \qquad
  \widehat{\sigma}^2_{\texttt{IC}} = \frac{1}{n_{\texttt{IC}} - 1}\sum_{k=1}^{n_{\texttt{IC}}}\left( \mathscr{F}\left( \bm{X}_k^{\texttt{IC}}; \widehat{\mathcal{V}} \right) - \widehat{\mu}_{\texttt{IC}} \right)^2 .
\end{align*}
The value of $\rho$ is directly related to the signal-triggering criterion and thus should be carefully chosen to achieve the desired IC average run length, ARL$_0$. Unlike the above three parameters, which are relatively straightforward to determine, the control limit is difficult to choose efficiently and properly. The foremost obstacle is that the distribution of the charting statistics is non-stationary, so it lacks analytical formula and is hard to estimate. A common way to estimate ARL$_0$ is to use computationally intensive methods such as Monte Carlo simulation. 

Given a pre-specified ARL$_0$, we propose an algorithm, Augmented Bootstrap (AB), to select the control limit $\rho$. The algorithm is based on the idea of bootstrapping but uses image transformation to expand the set of IC images. The enlarged sample size can facilitate a more effective and efficient estimation of ARL$_0$. Specifically, consider the following operations\footnote{We do not consider other rotational degrees (e.g., 30 degrees) in  the augmentation because they would require image extension to address the boundary problem.} on an image: (1) Rotate the image counter-clockwise by 90 degrees. (2) Rotate the image counter-clockwise by 180 degrees. (3) Rotate the image counter-clockwise by 270 degrees. (4) Flip the image horizontally. (5) Flip the image vertically. (6) Apply no transformation at all. 
Define $\mathcal{T}$ to be the collection of operations (1)-(5) and let $\mathcal{T}_0$ be the set of all six operations above. The selection process of $\rho$ is given by Algorithm 1 (the proposed AB algorithm). 
\begin{algorithm}
\caption{Augmented Bootstrap (AB) Algorithm} 
\begin{algorithmic}[1]
\State Given an ARL$_0$, estimate the range of $\rho\in(\underline{\rho},\bar{\rho})$, and initialize $\rho = (\underline{\rho} + \bar{\rho})/2$.
\While{$|\text{ARL}(\rho)-\text{ARL}_0|>\epsilon$} \Comment{$\epsilon$ is the desired precision}
\For {$r=1,2,\cdots, R$}  \Comment{$R$ is the simulation number}
 \State Set $b=1$.
 \Repeat 
  \State Randomly draw with replacement an IC image, $\bm{X}_b$, from $\mathcal{C}_0$ 
  \State Randomly apply an operation from $\mathcal{T}_0$ to $\bm{X}_b$, resulting in $\tilde{\bm{X}}_b$ 
  \State Compute $E_b = \max\left\{ 0, \lambda \left(\mathscr{F}\left(\tilde{\bm{X}}_b; \widehat{\mathcal{V}}\right) - \widehat{\mu}_{\texttt{IC}}\right) + (1 - \lambda) E_{b-1} \right\}$  \Comment{Recall $E_0=0$}
  \State $b \leftarrow b + 1$ 
 \Until{$E_b \geq \rho \sqrt{(\lambda \widehat{\sigma}^2_{\texttt{IC}}) /(2 - \lambda)}$}
 \State Let $\text{RL}(r) = b - 1$ 
\EndFor
\State Compute $\text{ARL}(\rho) = 1/R \sum_{r=1}^R \text{RL}(r)$. 
\If{$\text{ARL}(\rho) < \text{ARL}_0$}
    \State $\underline{\rho} \leftarrow \rho$ and $\rho \leftarrow (\underline{\rho} + \bar{\rho})/2$ 
\Else
    \State $\bar{\rho} \leftarrow \rho$ and $\rho \leftarrow (\underline{\rho} + \bar{\rho})/2$ 
\EndIf
\EndWhile
\end{algorithmic} 
\end{algorithm}

In the above algorithm, the bootstrap resampling on line 6 ensures the effectiveness of the algorithm in finding the corresponding $\rho$ given an ARL$_0$. Furthermore, the data augmentation on line 7 mitigates the possible issue of limited IC sample size (augmented by a factor 6), enhancing the algorithm efficiency. 
The number of run lengths simulated is governed by the choice of simulation number $R$, which is usually chosen to be a relatively large number. Lastly, we use a bisection search method on lines 15 and 18 for fast convergence because $\text{ARL}(\rho)$ is increasing in $\rho$.

%
%

\subsection{MODERN-Diagnosis: A Faulty Region Estimator}
\label{sec:3.3}

After MODERN-Chart gives a signal, the manufacturing process often needs to stop, and post-signal diagnosis is initiated. The diagnosis is critical for the ensuing repair process, so a quick and accurate isolation of the defective area can greatly reduce the equipment downtime. Motivated by this practical need, we propose a faulty region estimator, named MODERN-Diagnosis, to be used along with MODERN-Chart. The output of MODERN-Diagnosis is an ellipses-shaped mark on the OC image within which the defect is located; see Figure \ref{fig:sample} in Appendix \ref{app:dagm} for an example. The shape of the mark is not important, as we can easily adapt to cases where the faulty region is marked differently, e.g., using rectangular boxes.


Since MODERN-Diagnosis works on the same images that have been monitored by MODERN-Chart, it is built upon the same underlying CNN. Hence, to obtain our faulty region estimator, we re-purpose MODERN-Net $\mathscr{F}$ by making the following two changes: (i) We remove the logistic link from the last layer, and (ii) we reconstruct the last layer such that it outputs five numbers instead of a probability. The five numbers, denoted by $\bm{z} = (z_1, \cdots, z_5)^\prime$ are used to fully describe an ellipse in a plane: the \textit{x}- and \textit{y}-coordinates of the center, the lengths of the major and minor axes, and the angle between the major axis and the \textit{x}-axis. Denote the re-purposed network by $\mathscr{G}(\cdot; \mathcal{U})$ with $\mathcal{U}$ being the collection of learnable parameters in $\mathscr{G}$. Next, we estimate $\mathcal{U}$ by minimizing the following weighted $L_1$ loss: 
\begin{align}
  \label{eq:wtL1}
  \frac{1}{n_{\texttt{OC}}}\sum_{i = 1}^{n_{\texttt{OC}}} L_1\left(\bm{z}_i, \mathscr{G}\left(\bm{X}_i^{\texttt{OC}}; \mathcal{U}  \right) \right) =  \frac{1}{n_{\texttt{OC}}}\sum_{i = 1}^{n_{\texttt{OC}}} \bm{w}^\prime \left|\mathscr{G}\left( \bm{X}_i^{\texttt{OC}}; \mathcal{U} \right) - \bm{z}_i \right|,
\end{align}
where $\bm{z}_i$ denotes the five parameters of the ellipse in $\bm{X}_i^{\texttt{OC}}$, and $\bm{w} = (w_1, \cdots, w_5)^\prime$ are pre-specified weights representing the relative importance of the five parameters. For example, if both the location and the size of the faulty region are important, then an equal-weighted $L_1$ loss where $w_i = 1$ for $i = 1, \cdots, 5$ would be used; however, in some applications, the location of the defective area is more informative during the root cause investigation, and therefore the center coordinates in the loss function would receive more weights. 


Similar to $\mathscr{F}(\cdot; \mathcal{V})$, $\mathscr{G}(\cdot; \mathcal{U})$ is trained using stochastic gradient descent with variable learning rates. Different from the training of  $\mathscr{F}(\cdot; \mathcal{V})$, the training of $\mathscr{G}(\cdot; \mathcal{U})$ only uses the OC image set $\mathcal{C}_1$, which is much smaller than $\mathcal{C}_0$ in most manufacturing settings (defective products are usually in small quantity compared to non-defective ones). To address the issue, we apply the idea of feature reuse from the transfer learning literature (e.g., \citealt{zhuang2020comprehensive}). Specifically, in the stochastic gradient descent algorithm, the parameter values in $\mathcal{U}$ are initialized to be the same as those in $\widehat{\mathcal{V}}$ except for the last layer, which are randomly initialized. In this way, the knowledge learned by $\mathscr{F}(\cdot; \widehat{\mathcal{V}})$ is carried onto $\mathscr{G}(\cdot; \mathcal{U})$. Since there usually are commonalities in the image features used for IC/OC classification and those for fault isolation, via the above transfer learning technique, we can train MODERN-Diagnosis in a more efficient and effective manner. {\color{black}Finally, we remark that our MODERN framework is not limited to a specific CNN architecture. In our discussion, we adopt the Inception-ResNet architecture as it is widely regarded as state-of-the-art in the deep learning literature. However, the proposed monitoring and diagnostic methods can be carried out in the same manner with other CNN architectures. For illustration, refer to Appendix \ref{app:alex} for a comparison study using ``AlexNet''.  }



\subsection{Transfer Monitoring: Extending the Applicability of MODERN Framework}\label{sec:3.4}

Once the networks $\mathscr{F}(\cdot; \widehat{\mathcal{V}})$ and $\mathscr{G}(\cdot; \widehat{\mathcal{U}})$ are trained over image data from a certain manufacturing process, the proposed MODERN framework can be used for that process. However, it is not appropriate to directly apply the trained networks in a different manufacturing process, simply because the images observed in one setting may not be representative in the other and thus the results may be misleading. Moreover, it is often not practically feasible to sufficiently train the networks from scratch due to the lack of enough OC images as most manufacturers have relatively limited OC data. Thus, it would be desirable for managers to know whether and how the networks trained somewhere else can be applied in their own settings. 

In this subsection, we propose a two-step procedure to measure and extend the applicability of MODERN framework trained in a specific setting. The first step is \textit{applicability test}, where the direct applicability of our framework is quantitatively measured by a hypothesis testing approach; moreover, we provide a graphical tool for the ease of applicability assessment. If the applicability test is passed, then the second step is \emph{transfer monitoring}, where the knowledge learned by the pre-trained network is transferred to find the new control limit for the monitoring problem at hand. We describe these two steps in detail below. 

%

Recall that our network $\mathscr{F}(\cdot; \widehat{\mathcal{V}})$ is trained on IC/OC images from sets $\mathcal{C}_0$ and $\mathcal{C}_1$. In a different manufacturing setting, a possibly much smaller set of IC and OC images would typically be available as the training data. Let $\mathcal{S}_0 = \{\bm{S}_1^{\texttt{IC}}, \cdots, \bm{S}^{\texttt{IC}}_{m_0} \}$ be the set of $m_0$ IC images and $\mathcal{S}_1 = \{\bm{S}_1^{\texttt{OC}}, \cdots, \bm{S}^{\texttt{OC}}_{m_1} \}$ be the set of $m_1$ OC images in this new setting. We apply $\mathscr{F}(\cdot; \widehat{\mathcal{V}})$ to $\mathcal{S}_0$ and $\mathcal{S}_1$, and define two sets of outcome probabilities
%
%
\begin{align*}
  \widehat{\mathcal{P}}_0 = \left\{  \mathscr{F}\left( \bm{S}^{\texttt{IC}}_k; \widehat{\mathcal{V}} \right): 1 \leq k \leq  m_0 \right\}, \quad  \widehat{\mathcal{P}}_1 = \left\{  \mathscr{F}\left( \bm{S}^{\texttt{OC}}_k; \widehat{\mathcal{V}} \right): 1 \leq k \leq  m_1 \right\}.
\end{align*}
Then, a two-sample \textit{t}-test is used for our applicability test, with the test statistic being
\begin{align*}
  &t_{\texttt{app}} = \frac{\overline{\mathcal{P}}_1 - \overline{\mathcal{P}}_0}{\sqrt{\widehat{\sigma}^2_{\texttt{app}}\left( \frac{1}{m_0} + \frac{1}{m_1} \right)}}, \; \overline{\mathcal{P}}_0 = \frac{1}{m_0}\sum_{k=1}^{m_0} \mathscr{F}(\bm{S}_k^{\texttt{IC}}; \widehat{\mathcal{V}}), \; \overline{\mathcal{P}}_1 = \frac{1}{m_1}\sum_{k=1}^{m_1} \mathscr{F}(\bm{S}_k^{\texttt{OC}}; \widehat{\mathcal{V}}), \\
  &\widehat{\sigma}^2_{\texttt{app}} = \frac{\sum_{k=1}^{m_0}\left(\mathscr{F}(\bm{S}_k^{\texttt{IC}}; \widehat{\mathcal{V}}) - \overline{\mathcal{P}}_0  \right)^2 + \sum_{k=1}^{m_1}\left(\mathscr{F}(\bm{S}_k^{\texttt{OC}}; \widehat{\mathcal{V}}) - \overline{\mathcal{P}}_1  \right)^2}{m_0 + m_1 -2}.
\end{align*}
If the corresponding \textit{p}-value (based on the $t$ distribution with $m_0 + m_1 - 2$ degrees of freedom) is smaller than a pre-specified threshold (e.g., $0.05$), then we consider $\mathscr{F}(\cdot; \widehat{\mathcal{V}})$ capable of distinguishing images in $\mathcal{S}_1$ from those in $\mathcal{S}_0$ and thus applicable to the new manufacturing process. 
Finally, as for $\mathscr{G}(\cdot; \widehat{\mathcal{U}})$, its applicability can be directly assessed by visualizing $\mathscr{G}(\bm{S}^{\texttt{OC}}_k; \widehat{\mathcal{U}})$ for $\bm{S}^{\texttt{OC}}_k \in \mathcal{S}_1$. If the defective areas cannot be properly located, then more OC images are needed for further training. 

When the applicability test is passed for the network $\mathscr{F}(\cdot; \widehat{\mathcal{V}})$, MODERN-Chart should be re-constructed for the new manufacturing process. Particularly, the control limit should be updated as it may change for different settings even though the underlying network is deemed directly applicable. Indeed, our applicability test only measures the ability of $\mathscr{F}(\cdot; \widehat{\mathcal{V}})$ to separate $\widehat{\mathcal{P}}_1$ from $\widehat{\mathcal{P}}_0$. The statistical distributions of $\{\mathscr{F}(\bm{X}_k^{\texttt{IC}}; \widehat{\mathcal{V}}): \bm{X}_k \in \mathcal{C}_0 \}$ (old setting) and $\{\mathscr{F}(\bm{S}_k^{\texttt{IC}}; \widehat{\mathcal{V}}): \bm{S}_k \in \mathcal{S}_0 \}$ (new setting) may be quite different, rendering the original control limit invalid. To recompute the control limit, we enter the second step, transfer monitoring. Specifically, we replace the original image set $\mathcal{C}_0$ with the new image set $\mathcal{S}_0$ and call our augmented bootstrap algorithm again to find the parameter $\rho$. Notably, with augmented bootstrap, we only need a few hundred images to achieve a commonly used ARL$_0$ values (e.g., 200, 500 or 1000). This is much smaller sample size requirement in comparison with tens of thousands of images typically needed for sufficiently training a deep neural network. Therefore, due to its efficiency and practicality, transfer monitoring substantially extends the applicability of our MODERN framework. 

Finally, we remark that the two-step procedure described above does not depend on the specific original training set. It could be from a real-life manufacturing process, or just a simulated/synthetic data set. In this paper, we trained our networks $\mathscr{F}(\cdot; \widehat{\mathcal{V}})$ and $\mathscr{G}(\cdot; \widehat{\mathcal{U}})$ over a benchmark dataset DAGM (to be introduced in Section \ref{sec:5.1} later), which includes a wide variety of industrial textures and defect types. Through the applicability test and transfer monitoring, we can apply MODERN framework to both simulated and real-life settings; see Section \ref{sec:5} for details.

\section{Asymptotic Properties}\label{sec:4}

In addition to its practical relevance, our proposed MODERN framework also has strong theoretical properties, which in turn offers useful managerial insights for practitioners. {\color{black}In Subsection \ref{sec:4.1}, by adapting some recent theoretical techniques in fully-connected neural networks for convolutional neural networks, we establish the minimax asymptotics for defect likelihood estimation and fault diagnosis. In Subsection \ref{sec:4.2}, we show that the AB algorithm in transfer monitoring achieves the convergence rate of $\sqrt{n}$.}

\subsection{Theoretical Results in Estimation and Diagnostics}\label{sec:4.1}
In this subsection, we show the consistency of the two networks $\mathscr{F}(\cdot; \mathcal{V})$ and $\mathscr{G}(\cdot; \mathcal{U})$, characterize their respective rates of convergence, and discuss the implications of the derived asymptotic properties. First, we show that $\mathscr{F}(\cdot; \mathcal{V})$ is able to estimate the OC likelihood with increasing accuracy (i.e., the estimation error approaches $0$) as the training set grows larger. Consider an image $\bm{X}$ and its binary label $y$. We assume that the conditional distribution $y| \bm{X}$ is Bernoulli with the probability of the image being OC equal to $\mathbb{P}(f^*(\bm{X}))$, where $\mathbb{P}(x) = 1/(1 + \exp(-x))$ is the logistic function. Thus, $f^*$ represents the log-odds of the true OC probability, and our network after training, $\mathcal{F}(\cdot; \widehat{\mathcal{V}})$, is trying to approximate the function $\mathbb{P}(f^*)$. {\color{black}This logistic regression setup ensures that the OC/IC probability is always bounded in $(0, 1)$. Therefore, as the total sample size grows, we accumulate both IC and OC images.} Let $\nu$ denote the distribution of $\bm{X}$, and we assume that there exists a positive constant $B$ such that the support of $\nu$ is included in $[-B, B]^d$, where $d$ is the dimension (i.e., resolution) of $\bm{X}$. Moreover, let $\mathcal{D}$ be the number of layers and $\mathcal{W}$ be the maximal number of filters of all layers in our CNN. 
Then, the following theorem establishes the asymptotic consistency of MODERN-Net and identifies its convergence rate for estimating $f^*$. 

\begin{theorem}
\label{ester}
%
Assume the following conditions hold: (i) $f^{*}({x})$ is Lipschitz continuous; (ii) $\mathcal{D}\mathcal{W} = \mathcal{O}(N^{d/(2d+4)})$; and (iii) $(\mathcal{D}\mathcal{W})^2\log (\mathcal{D}\mathcal{W})< Ne^{-1}$, where $N = n_{\texttt{IC}} + n_{\texttt{OC}}$ is the training sample size and $e$ is Euler's constant. Then we have 
\begin{equation*}
\mathbb{E} \left[\left\| \log\left(\frac{\mathscr{F}(\cdot; \widehat{\mathcal{V}})}{1 - \mathscr{F}(\cdot; \widehat{\mathcal{V}})} \right) - f^*\right\|_{L^2(\nu)}^2 \right] = \mathcal{O} \left( N^{-2/(2+ d)}\log^{3/2}N \right).
\end{equation*}
\end{theorem}

Our next result shows that the faulty region estimator $\mathscr{G}(\cdot; \mathcal{U})$ (i.e., MODERN-Diagnosis) is also consistent. Given an OC image $\bm{X}^{\texttt{OC}}\in\mathbb{R}^d$, let its defective area be characterized by an ellipse $\bm{z}\in\mathbb{R}^5$; moreover, suppose the true functional relationship between the image and its faulty region is given by $g^*: \mathbb{R}^d \rightarrow \mathbb{R}^5$. Therefore, our trained network $\mathcal{G}(\cdot; \widehat{\mathcal{U}})$ aims to approximate the function $g^*$. Recall that the network is trained by minimizing the \emph{sample} weighted $L_1$ loss given by \eqref{eq:wtL1}. Let $\mathcal{L}(g^*) = \mathbb{E}\left\{ \bm{w}' |g^*(\bm{X}^{\texttt{OC}}) - \bm{z}| \right\}$ denote the weighted $L_1$ loss at the \emph{population} level. Then, the following theorem gives the convergence rate for both the loss function of our network and the faulty region estimator, establishing the asymptotic consistency.  


%
\begin{theorem}
  \label{esterlad}
  Suppose that $g^{*}({x})$ is Lipschitz continuous. 
  \begin{enumerate}[(i)]
  \item If we assume the conditions in Theorem \ref{ester} with $N$ replaced by $n_{\texttt{OC}}$, then we have
    \begin{equation*}
      \mathbb{E} \left[\left| \mathcal{L}(\mathscr{G}(\cdot; \widehat{\mathcal{U}})) - \mathcal{L}(g^*) \right| \right] = \mathcal{O}\left( n_{\texttt{OC}}^{-1/(2+ d)}\log^{3/2}n_{\texttt{OC}} \right).
    \end{equation*}
  \item If we further assume that there exists $a_1, a_2 > 0$ such that for any $|c| \leq a_1$,
    \begin{equation*}
      \left|F_{\bm{z} \mid \bm{X}^{\texttt{OC}}}\left( g^*(x)+c \right) - F_{\bm{z} \mid \bm{X}^{\texttt{OC}}} \left( g^*(x) \right) \right| \geq a_2|c|,  \text { a.s. },
    \end{equation*}
    where $F_{\bm{z} \mid \bm{X}^{\texttt{OC}}}(\cdot)$ is the cumulative distribution function of $\bm{z}$ given $\bm{X}^{\texttt{OC}}$, then we have
    \begin{align*}
      \mathbb{E} \left[ \left\| \mathscr{G}(\cdot; \widehat{\mathcal{U}}) - g^* \right\|_{L^2(\nu)}^2 \right] = \mathcal{O}\left( n_{\texttt{OC}}^{-1/(2+ d)}\log^{3/2}n_{\texttt{OC}} \right). 
    \end{align*}
  \end{enumerate}
\end{theorem}

\textbf{Theoretical implications.} The convergence rate $\mathcal{O}({N}^{-2/(2+d)})$ established in Theorem \ref{ester} achieves the minimax (robust) optimal rate up to a log factor in nonparametric estimation of $d$-dimensional Lipschitz targets \citep{korostelev2012minimax}. It means that our MODERN-Net makes quite efficient use of the training images. Notably, our neural network, constructed as a nonlinear approximation to the target $f^*$, enjoys the same optimal convergence rate as those linearizing methods (e.g., splines) in the conventional setting. In Theorem \ref{esterlad}, the convergence rate of $\mathcal{G}$ under the functional $\mathcal{L}$ is of interest, because the parameter $\widehat{\mathcal{U}}$ is obtained by $L_1$ loss minimization. Note that the weighted $L_1$ loss function is used because of its numerical robustness to outliers. 
%
%
The additional assumption in Theorem \ref{esterlad}(ii) is a mild condition. It indicates that there exists a neighborhood around $g^*(x)$ in which the perturbation of the conditional probability distribution function $F_{\bm{z}|\bm{X}^{\texttt{OC}}}$ is larger than the fluctuation of  $g^*(x)$. Similar calibration  assumptions are common in the statistics literature (\citealt{belloni2011L1, madrid2022risk}). 

\textbf{Practical implications - part 1.} The results from the above theorems show that our estimation errors approach $0$ as the sample size grows, which confirms the benefits from acquiring additional training images. Notably, the convergence rates identified in both theorems become slower as $d$ (i.e, the dimension of the images, which is equivalent to the image's resolution) increases. 
In the context of image data analysis, intuition may suggest that images with higher resolution are always preferred and thus more advanced imaging devices are always beneficial. However, a meaningful and interesting practical implication of our results is that upgrading the monitoring equipment (e.g., cameras), which is usually for generating higher-resolution images, may not help enhance the accuracy of our MODERN framework. This counter-intuitive finding can be explained by the nature of deep-learning models in computer vision: with higher resolution, there are a greater number of image patterns or features for the model to sift through, making it more challenging to uncover the patterns associated with IC/OC classification. {\color{black}In practice, does this mean that manufacturers should never upgrade their equipment? Moreover, should they even downgrade their devices to get images of low resolutions? The answer is indeed yes according to Theorem \ref{ester} and Theorem \ref{esterlad}. We note, however, an important question arises in such a scenario concerning whether we can still tell the OC images apart from those IC images, with each image having only a few pixels. That is, are we still able to label the images correctly? Motivated by this question, we analyze the impact of mislabeling on MODERN-Net's performance. }

{\color{black}
Suppose that the true binary label $y$ of an image $\bm{X}$ may be contaminated. Denote the observed but contaminated version by $\tilde{y}$. Formally, we have 
 $$
\tilde{y}=
\left\{
  \begin{array}{ll}
    y, & \mathrm{with \quad probablity} \quad 1-\delta; \\
    1-y, & \mathrm{with \quad probability} \quad \delta.
  \end{array}
\right.
$$
As a result, the training data are $\{(\bm{X}_i, \tilde{y}_i): 1 \leq i \leq N \}$, and our neural network $\tilde{\mathcal{F}}(\cdot;\hat{\mathcal{V}})$ is obtained by minimizing the contaminated cross-entropy loss
 \begin{align*}
 \frac{1}{N}\sum_{i=1}^N \tilde{L}\left( \tilde{y}_i, \mathscr{F}(\bm{X}_i; \mathcal{V}) \right) = -\frac{1}{N}\sum_{i=1}^N \Big\{ \tilde{y}_i\log\left( \mathscr{F}(\bm{X}_i; \mathcal{V})\right) + (1 - \tilde{y}_i) \log\left( 1 - \mathscr{F}(\bm{X}_i; \mathcal{V}) \right) \Big\}.
\end{align*}
Our next theorem quantifies the impact from mislabeling. 
\begin{theorem}
\label{esternoisy}
%
Assume that the conditions in Theorem \ref{ester} hold. We have 
\begin{equation*}
\mathbb{E} \left[\left\| \log\left(\frac{\tilde{\mathscr{F}}(\cdot; \widehat{\mathcal{V}})}{1 - \tilde{\mathscr{F}}(\cdot; \widehat{\mathcal{V}})} \right) - f^*\right\|_{L^2(\nu)}^2 \right] = \mathcal{O} \left( N^{-2/(2+ d)}\log^{3/2}N \right)+\mathcal{O}\left(\delta\right).
\end{equation*}
\end{theorem}
}

{\color{black}\textbf{Practical implications - part 2}. Theorem \ref{esternoisy} suggests that, if the mislabeling rate exceeds $N^{-2/(2+ d)}\log^{3/2}N$, it starts to degrade the performance of the neural network. On the other hand, if the mislabeling rate is below or at the order of $N^{-2/(2+ d)}\log^{3/2}N$, the asymptotic performance of the neural network would be intact. Now based on Theorem \ref{ester} -- \ref{esternoisy}, our practical guidelines for manufacturers are as follows. It would be beneficial for them to work with low-resolution images by not upgrading or even downgrading their imaging devices, provided that the images are labeled at a high-fidelity level, with the asymptotic bound set at $N^{-2/(2+ d)}\log^{3/2}N$, for both IC/OC classification and faulty region isolation. In the case where existing cameras cannot render satisfactory labeling accuracy (i.e., the mislabeling rate exceeds the asymptotic bound), it becomes necessary to collect higher-resolution images to allow precise labeling. Investments in equipment upgrades are then justified. }

{\color{black}
\subsection{Theoretical Results in Transfer Monitoring}
\label{sec:4.2}
Our primary result in transfer monitoring is for quantifying the sample size reduction. Once it is determined that our trained network is applicable in the new setting, we only need to utilize the IC image set with our AB algorithm for searching a new control limit. In other words, the size of the IC set is tied to the efficiency of the AB algorithm. Our next theorem establishes the convergence rate of the algorithm in terms of the IC sample size. 

Before formally stating the theorem, let us introduce some notations. For any given $\rho$, let $\mathrm{ARL}(\rho)$ and $\widehat{\mathrm{ARL}}(\rho)$ denote the true $\mathrm{ARL}$ on the population level  and the $\mathrm{ARL}$ determined by the AB algorithm, respectively. Let $F_0$ denote the cumulative distribution function (CDF) of $\log\left(\mathscr{F}(\bm{X}; \mathcal{V})/(1 - \mathscr{F}(\bm{X}; \mathcal{V}) \right)$ when $\bm{X}$ is IC. Let $\mathscr{S}$ denote the collection of finite-variation functions that satisfy the run length regularity condition \eqref{eq:transfer-condition} in the appendix. 
\begin{theorem}
    \label{thm:transfer}
    Assume that $\mathscr{S}$ contains an open neighborhood of $F_0$ and the origin where the distance is defined by the $L^\infty$-norm. Then we have 
    \begin{align*}
        \left| \widehat{\mathrm{ARL}}(\rho) - \mathrm{ARL}(\rho) \right| = \mathcal{O}_p\left( \frac{1}{n^{1/2}_{\texttt{IC}}}\right), 
    \end{align*}
    where $\mathcal{O}_p(\cdot)$ denotes the convergence in probability.
\end{theorem}

\textbf{Remark.} It can be seen that the AB algorithm achieves the typical convergence rate of $n^{-1/2}_{\texttt{IC}}$, which is much faster than $N^{-2/(2 + d)}\log^{3/2}N$, the convergence rate in Theorem \ref{ester}, given that $d$ is often in the range of tens of thousands for images. Indeed, our transfer monitoring technique estimates a single number instead of a high-dimensional distribution, allowing significant sample size reduction. Theorem \ref{thm:transfer} provides the theoretical basis for this useful property. 
}


\section{Numerical Studies}
\label{sec:5}

In this section, we conduct a series of numerical studies to demonstrate the performance of the proposed MODERN framework in different settings and data environments. {\color{black}First, we conduct numerical experiments regarding the impact of image resolution on MODERN-Net's performance.} Second, we evaluate MODERN-Chart and MODERN-Diagnosis using a benchmark dataset called DAGM. Then, we compare our method with a state-of-the-art approach in simulated experiments. Finally, we apply the proposed method in {\color{black}two manufacturing settings (commutator manufacturing and leather manufacturing)} and test it on the real-life data. 

{\color{black}
\subsection{The Impact of Image Resolution}
\label{sec:5.0}

Our theoretical results in Subsection \ref{sec:4.1} suggest that increasing image resolutions can be detrimental to our neural network's performance in the asymptotic setting. In this subsection, we analyze the impact of image resolution with finite samples. We first simulate both IC and OC images in the low-resolution ($128\times 128$), medium-resolution ($256\times 256$) and high-resolution ($512\times 512$) scenarios. Then we train our MODERN-Net and evaluate its classification accuracy in the three scenarios, respectively. Specifically, all IC image pixel intensities are generated from I.I.D. standard normal distributions. The OC images are generated in the same way except that the intensity level is shifted to a fixed value for the five pixels in the center of the image. Figure \ref{fig:resolution} shows an IC and OC image in our experiment. 

\begin{figure}[htp]
  \centering
  \caption{Simulated IC (left) and OC (right) image}
  \label{fig:resolution} 
  \begin{tabular}{cc}
    \hspace{-0.75cm}\includegraphics[width = 0.25\textwidth]{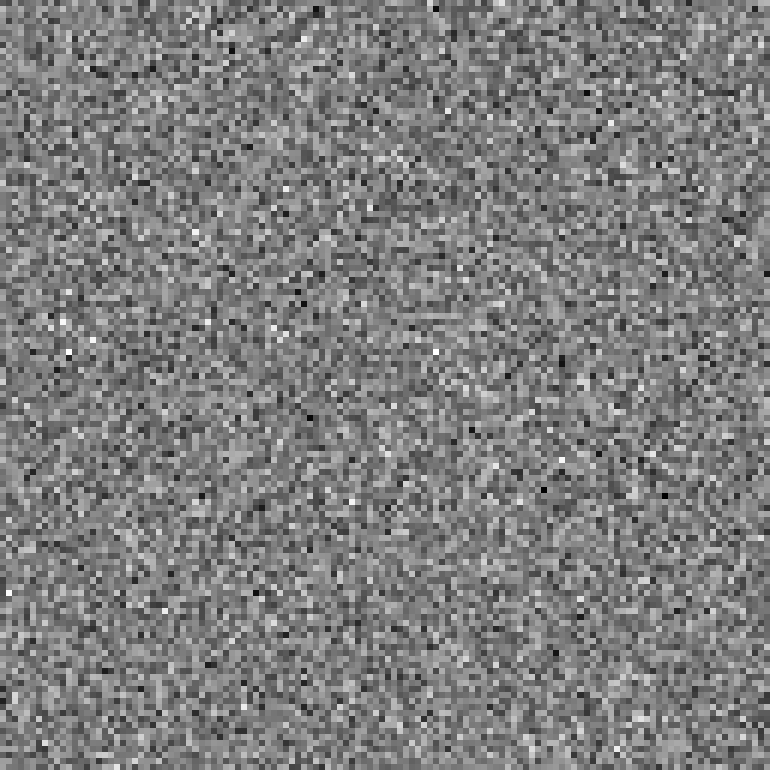} &
    \includegraphics[width = 0.25\textwidth]{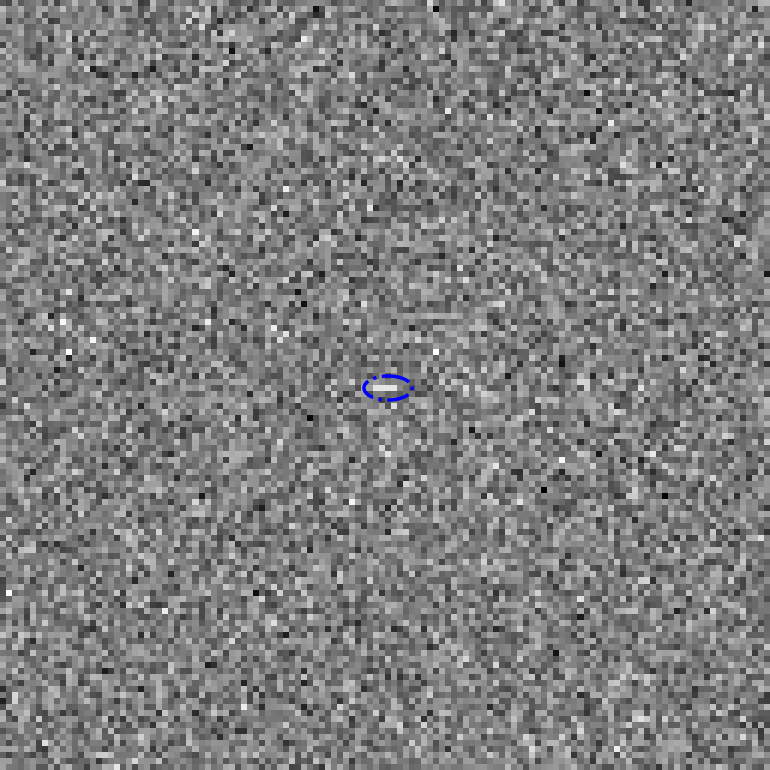} 
  \end{tabular}
  \fignote{The blue circle in the OC image indicates the pixel locations where the intensity levels have been shifted.}
\end{figure}

For each resolution level, we generate 100 images for training, validation and testing, respectively, with 50\% of all the images being IC and others being OC. The classification accuracy is defined to be the proportion of correctly classified images in the testing set. The shift size of 1, 2 and 3 are considered. The simulation results are summarized in Table \ref{tab:resolution}. It can be seen from the table that, as the shift size grows (i.e., the OC signal strengthens), the neural network in the low-resolution scenario is able to detect the signal quite well. However, this is not the case in the high-resolution scenario, consistent with our theoretical findings. 

\begin{table}[ht]
  \centering
  \caption{The impact of image resolution on classification accuracy.}
  \label{tab:resolution}
  \begin{tabular}{cccc}
    \hline
    Shift size& Low  & Medium & High \\ \hline
      1        & 0.46 & 0.47   & 0.47 \\
      2        & 0.59 & 0.51   & 0.46 \\
      3        & 1.00 & 0.96   & 0.46 \\
      \hline
  \end{tabular}  
  \vspace{2pt}
   \fignote{Entries in the table are the proportion of correctly classified images.}
\end{table}

}

\subsection{Performance Evaluation Using DAGM}
\label{sec:5.1}

\subsubsection*{DAGM: A Benchmark Dataset.} 
DAGM is a dataset created for the industrial image processing competition at the 29th Annual Symposium of the German Association for Pattern Recognition. After the competition, it has been used as a benchmark dataset in academic research (e.g., \citealt{weimer2016design, wang2018fast, zhang2023prototypical}). The images are artificially generated by a number of different texture and defect models. Although synthetic, the images are quite similar to real world textures of manufactured materials. The dataset consists of 10 classes of textured surfaces, encompassing a good range of industrial applications. In classes 1 through 6, each contains 1000 IC images and 150 OC images. In classes 7 through 10, each contains 2000 IC images and 300 OC images. 
The defect in each OC image is labeled by a surrounding ellipse. 
Each image has $512\times 512$ pixels and the dataset is publicly available from the Heidelberg Laboratory for Image Processing at the University of Heidelberg (\url{https://hci.iwr.uni-heidelberg.de/content/weakly-supervised-learning-industrial-optical-inspection}). 
We provide a sample of DAGM images in Appendix \ref{app:num} to show what the images look like. From the sample, we can observe that some surface defects are hard for human eyes to discern; moreover, all these background textures contain certain random patterns (i.e., no gold standard), and some of them clearly do not meet the locality and stationarity assumptions of the Markov Field approach. Therefore, this dataset serves as a perfect benchmark to evaluate the performance of our proposed method for quality monitoring and diagnosis.

\subsubsection*{Numerical Setup and Results.}
We divide DAGM dataset into training (60\%), validation (25\%), and testing (15\%) sets. Our model is calibrated using the training and validation sets. Its performance is examined using the testing set. 
First, we augment the training set to enlarge the data size. Recall that $\mathcal{T}$ is the set of five image transformations defined in Section \ref{sec:3.2}. For each OC image in the training set, apply all five transformations in $\mathcal{T}$, whereas for each IC image in the training set, apply randomly one of the transformations in $\mathcal{T}$. 
This results in 16,850 IC images and 7410 OC images in the augmented training set. 
Next, we train MODERN-Net $\mathscr{F}(\cdot; \mathcal{V})$ with the following implementation details. The cross-entropy loss is minimized using stochastic gradient descent with batch size $N_k = 16$. The variable learning rate schedule is implemented with $\alpha_L = 0.001$, $\alpha_U = 0.01$, and $C = 2000$. The validation set is used for determining when to terminate the iterations in the stochastic gradient descent algorithm. Stopping the iterations too early results in underfitting while letting the algorithm run for too long causes overfitting. Three measures are calculated on the validation set as the number of epochs increases: the proportions of correct classification of OC images (sensitivity), IC images (specificity) and the average of the two. 
We observe from the training progression that the validation performance stabilizes after 120 epochs, so we stopped the training process at the 150-th epoch; see details in Appendix \ref{app:dagm}. 
Lastly, we examine the monitoring performance of our control chart using the testing set. Let $t$ be the time index and $t=1$ denotes the beginning of the monitoring. For each class of DAGM, we randomly draw with replacement from the testing set an IC image for $t = 1, \cdots, 20$ and OC image for $t \geq 21$. After an image is drawn, randomly apply one transformation in $\mathcal{T}_0$ to the image and compute the charting statistic by \eqref{eq:1} and \eqref{eq:2}. This process continues until the control chart delivers a signal. Therefore, $t=21$ is the change point when the process becomes OC and the time point of the signal is the control chart's OC run length.  We repeat this process 100 times and report the OC average run length (ARL$_1$) in Table \ref{tab:1}. ARL$_0$ is set equal to 1,000 and the control limit is found to be $0.0275$ by the AB algorithm\footnote{ {\color{black} Our numerical experiment can be done similarly given a different ARL$_0$ value. For instance, a smaller value (e.g., 500) would lead to a lower control limit, rendering that the control chart signals an alarm every 500 observations (instead of 1,000) while the process is IC. For simplicity, here we only report the results for ARL$_0=$ 1,000.}}. We also report in Table \ref{tab:1} the proportions of runs our chart signals before and exactly at the change point (the column ``Prop.Early'' and ``Prop.On'' respectively). Results show that the proposed chart gives a signal right at the change point most of the time. 

\begin{table}[ht]
  \centering
  \caption{The numerical performance of the proposed method on the DAGM testing set.}
  \label{tab:1}
  \begin{tabular}{ccccccc}
    \hline
    Class & ARL$_1$ & sdARL$_1$ & Prop.Early & Prop.On & SDSC  & sdSDSC \\ \hline
      1   & 20.64   & 0.2533   & 0.02        &  0.98    & 0.8293 & 0.0280 \\
      2   & 20.96   & 0.0400   & 0.01        &  0.99    & 0.7157 & 0.0270 \\
      3   & 21.00   & 0.0000   & 0.00        &  1.00    & 0.7974 & 0.0131 \\
      4   & 20.82   & 0.1218   & 0.03        &  0.97    & 0.7696 & 0.0348 \\
      5   & 21.00   & 0.0000   & 0.00        &  1.00    & 0.8719 & 0.0135 \\
      6   & 21.00   & 0.0000   & 0.00        &  1.00	 & 0.6548 & 0.0425 \\
      7   & 20.45   & 0.2618   & 0.06        &  0.94    & 0.8820 & 0.0133 \\
      8   & 21.01   & 0.0100   & 0.00        &  0.99    & 0.6036 & 0.0336 \\
      9   & 20.88   & 0.1104   & 0.02        &  0.98    & 0.8352 & 0.0101 \\
     10   & 20.88   & 0.0844   & 0.02        &  0.98    & 0.7072 & 0.0159 \\ 
    \hline
  \end{tabular}
  \vspace{8pt}
  \fignote{Column sdARL reports the standard error of ARL$_1$ of the 100 repeated simulations. Column sdSDSC shows the standard error of the average SDSC in each class.} 
\end{table}

Now, we turn to evaluate the fault isolation performance -- how close MODERN-Diagnosis can identify the ellipse defective area in OC images? We use the Sorensen-Dice similarity coefficient (SDSC) as the performance metric: 
\begin{equation}
\text{SDSC} = \frac{2 \times \mbox{Area}( \mbox{EstimatedRegion} \cap \mbox{TrueRegion} ) }{\mbox{Area}( \mbox{EstimatedRegion} ) + \mbox{Area}(\mbox{TrueRegion})}.  \non
\end{equation}
The metric is between $0$ and $1$, with 1 being the case of perfect estimation. 

To begin the evaluation, we first train the network $\mathscr{G}(\cdot; \mathcal{U})$ on the OC images from the previously augmented training set. Note that even though the training set is expanded, the number OC images is still relatively small. To address this issue and to save training time, we utilize the following transfer learning technique. The parameter values in $\mathcal{U}$ are initialized to be the same as those in $\widehat{\mathcal{V}}$ except for the last layer, which are randomly initialized (recall that last fully connected layer of $\mathscr{G}$ has 5 outputs while $\mathscr{F}$ has only 2 in the last layer). Here the idea is to transfer the knowledge learned by $\mathscr{F}(\cdot; \widehat{\mathcal{V}})$ to $\mathscr{G}(\cdot; \mathcal{U})$ because those image features crucial for IC/OC classification should also be important for fault isolation. Then $\widehat{\mathcal{U}}$ is obtained by minimizing the weighted $L_1$ loss with the weight $\bm{w}$ chosen to be $(2, 2, 1, 1, 1)^\prime$; see equation \eqref{eq:wtL1}. Assigning more weight to the center coordinates prioritizes fault location over fault size. 
In our training progression, we observe that the validation SDSC stabilizes after 250 epochs, so we stopped the training process at the 275-th epoch; see details in Appendix \ref{app:dagm}

Then, we apply $\mathscr{G}(\cdot; \widehat{\mathcal{U}})$ to all the test OC images in each class and report the average SDSC and its standard error in Table \ref{tab:1}. Our findings suggest that $\mathscr{G}(\cdot; \widehat{\mathcal{U}})$ is able to estimate the faulty region reasonably well. To further visualize the fine performance of our faulty region estimator, we show in Figure \ref{fig:SDSC} the OC images corresponding to the median SDSC in each class.
The figure clearly shows that, even though the estimated ellipses do not fully overlap with the true ones, the defective areas have all been accurately highlighted, which is sufficient for practical use. 
\begin{figure}[htp]
  \centering
  \caption{Performance of MODERN-Diagnosis: Estimated vs true faulty regions.}
  \label{fig:SDSC} 
  \begin{tabular}{ccccc}
    \hspace{-0.75cm}\includegraphics[width = 0.19\textwidth]{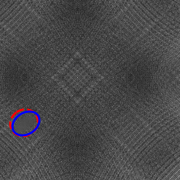} &
    \includegraphics[width = 0.19\textwidth]{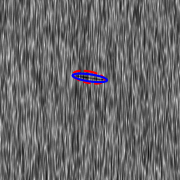} &
    \includegraphics[width = 0.19\textwidth]{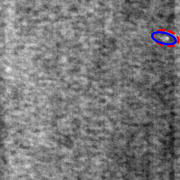} &
    \includegraphics[width = 0.19\textwidth]{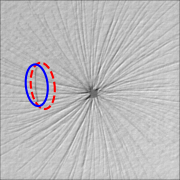}&
    \includegraphics[width = 0.19\textwidth]{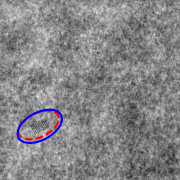} \\
    \hspace{-0.75cm}\includegraphics[width = 0.19\textwidth]{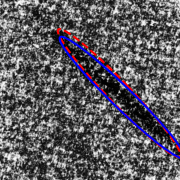} &
    \includegraphics[width = 0.19\textwidth]{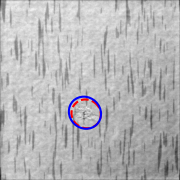} &
    \includegraphics[width = 0.19\textwidth]{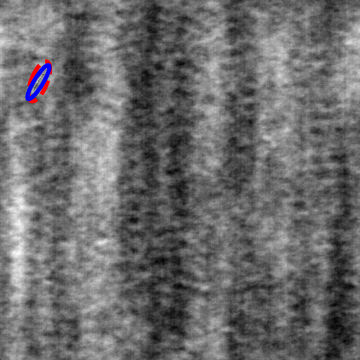} &
    \includegraphics[width = 0.19\textwidth]{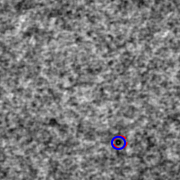}&
    \includegraphics[width = 0.19\textwidth]{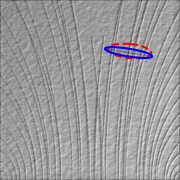}
  \end{tabular}
  \vspace{8pt}
  \fignote{The blue solid and red dashed ellipses represent the true and estimated faulty regions, respectively. }
\end{figure}
%

\subsection{Comparison with Markov Field Method}
\label{sec:5.3}

In this subsection, we compare our MODERN framework with a state-of-the-art image-based quality control method, namely, the Markov Field (MF) approach suggested by \citet{bui2018monitoring}. While MODERN framework can deal with images of any kind, the MF method is designed particularly for images that satisfy the locality and stationarity conditions; i.e., MF works the best for those images. Hence, to have a fair comparison, we apply both methods to a set of simulated images with the above Markov Field properties. In this way, we can also demonstrate the performance of the proposed MODERN framework with synthetic data. The images in our simulated example are generated by the following auto-regressive (AR) model:
\begin{align*}
  \bm{X}(i, j) = \phi_1 \bm{X}(i - 1, j) + \phi_2 \bm{X}(i, j - 1) + \varepsilon(i, j), \quad 1 \leq i, j \leq n ,
\end{align*}
where $\bm{X}(i, j)$ is the intensity level of image $\bm{X}$ at the $(i, j)$-th pixel, $\phi_1$ and $\phi_2$ are the AR coefficients, $n=512$ represents the image resolution, and $\{ \varepsilon(i, j) \}$ are i.i.d. random noises with the distribution $N(0, \sigma^2)$. When the production process is IC, we set $\phi_1 = 0.65$, $\phi_2 = 0.35$ and $\sigma = 0.1$. 
We consider two types of defects in the center area where pixel intensities are also generated by the AR model. Specifically, type-1 defect has $\phi_1 = 0.65$, $\phi_2 = 0.35$ and $\sigma = 10^{-6}$, whereas type-2 defect has $\phi_1 = \phi_2 = 0$ and $\sigma = 0.1$. Thus, type-1 defect is a change in the noise level and type-2 defect is a change in the correlation. Note that, since our MODERN-Net is trained on the benchmark dataset DAGM, we need to employ the transfer monitoring technique (introduced in Section \ref{sec:3.4}) before we apply the trained network to the simulated AR images. To that end, we conduct the applicability test first and the result shows that MODERN-Net is capable of distinguishing the IC and OC (of both types) images quite well even though it has not been trained on any of these AR images. 
For visual illustrations of the IC/OC images from the AR model and the applicability test, refer to Appendix \ref{app:sim} for details.

For each method, we simulate the monitoring process with the change point occurring at $t=21$ (i.e., feeding the first OC image as the 21-st input) and repeat the simulation 100 times. Table \ref{tab:2} summarizes the results concerning the ARL$_1$, where ARL$_0$ is set equal to $200$ for both methods. Here, we focus on comparing the timeliness of signal delivery, so we discard those runs in which a signal is given before the change point in the repeated simulations (hence only ``Prop.On'' is shown in the table and ``Prop.Early'' is excluded). 
From Table \ref{tab:2}, it is clear that our method delivers the signal in a much timelier manner than the MF approach does. In addition, the performance of MODERN is stable, with the standard error of ARL$_1$ being zero. By contrast, the MF method almost never signals at the change point and the ARL$_1$ is large, meaning the cost of delayed signaling would be high. 
\begin{table}[ht]
  \centering
  \caption{Comparison between MODERN framework and the MF method in simulated experiments.}
  \label{tab:2}
  \begin{tabular}{ccccccc}
    \hline
    Fault Type     & Method  & ARL$_1$   & sdARL$_1$ &  Prop.On & SDSC      & sdSDSC \\ \hline
                   & MODERN  & 21.00     & 0.0000   &  1.00     & 0.6286    & 0.0084 \\
        1          & MF      & 153.75    & 6.7987   &  0.01     & 0.0000    & 0.0000 \\
                   & {\color{black}t-stat}  & {\color{black}$19.5258^{**}$} &          &           & {\color{black}$74.8333^{**}$} &  \\
                   & MODERN  & 21.00   & 0.0000   &  1.00     & 0.6208 & 0.0057 \\
        2          & MF      & 150.58  & 6.3956   &  0.00     & 0.5762 & 0.0097  \\
                   & {\color{black}t-stat}   & {\color{black}$20.2608^{**}$} &          &           & {\color{black}$3.9642^{**}$} &   \\
    \hline
  \end{tabular}
  \vspace{8pt}
  \fignote{{\color{black} The row of \textit{t-stat} shows the t-statistics that compare the two methods, and $\;^{**}$ means p-values $<.01$.}}
\end{table}

Next, we compare the fault isolation performance by respectively applying MODERN-Diagnosis and MF method to 100 simulated OC images. The average values of SDSC and their standard errors are shown in Table \ref{tab:2}. Several observations from the results are noteworthy. First, our proposed method isolates the fault reasonably well for both types of defect. The SDSC metric is reasonably high and the standard error is quite small. Second, the MF approach is better at locating type-2 fault (change in correlation) than type-1 fault (change in noise level). In fact, it fails to detect any defective pixels for any type-1 OC image. Lastly, the proposed method outperforms the MF approach in both cases. For type-1 OC images, the MF method is dominated by MODERN-Diagnosis, whereas for type-2 defect both methods can achieve comparable performances. In Appendix \ref{app:sim}, we provide representative visual results for fault isolation given by the two methods, confirming our observations above.

\subsection{Application to Manufacturing of Electric Commutators}
\label{sec:5.4}

Electric commutators are an integral part of every electric motor that powers windshield wipers or washing machines. Quality monitoring and diagnosis are critical in their manufacturing process, which has entered the Industry 4.0 era. Recognizing such an important application area, we demonstrate our MODERN framework in the context of electric commutator manufacturing based on real-life images. These images are made available by Kolektor Mobility, a global commutation system supplier (\url{https://www.kolektor.com/commutators}). Our sample includes 14 IC images and 1 OC image with a surface defect. 
Before we use the trained MODERN-Net in this specific setting, we conduct the applicability test as in Section \ref{sec:3.4}. Both test statistic value and graphical tool indicate that the network $\mathscr{F}(\cdot; \widehat{\mathcal{V}})$ is able to separate the OC image from the others in the manufacturing process considered here. Representative electric commutator images and the details of the applicability test are provided in Appendix \ref{app:steel}. 
 


We setup the quality monitoring and diagnosis process as follows. From all the IC images, ten are used to determine the charting parameters, including the control limit. The remaining images are used for monitoring, where the first three images are IC, and the OC image is input from the fourth time point onward. Because of the small sample size, we let IC average run length ARL$_0 = 25$, and the control limit is found to be $0.708$ by the AB algorithm. 

Then, we apply the proposed MODERN framework to monitor the process. Including the 10 IC images used in control limit calculation, we plot MODERN-Chart for 16 time points; moreover, the fault isolation is given by MODERN-Diagnosis. The results are visualized by Figure \ref{fig:steel}(a). The figure shows that our chart delivers the signal at the exact change point (i.e, $t=15$). For the identified faulty region, although the ellipse does not surround the entire defective area, it flags the part requiring attention, providing helpful information for the quality personnel investigating the root cause. 
\begin{figure}[ht!]
\centering
\caption{Comparison between MODERN framework and MF method in electric commutator manufacturing.} \label{fig:steel}
\begin{subfigure}[t]{\linewidth}
\centering
			 \includegraphics[width = .32\linewidth]{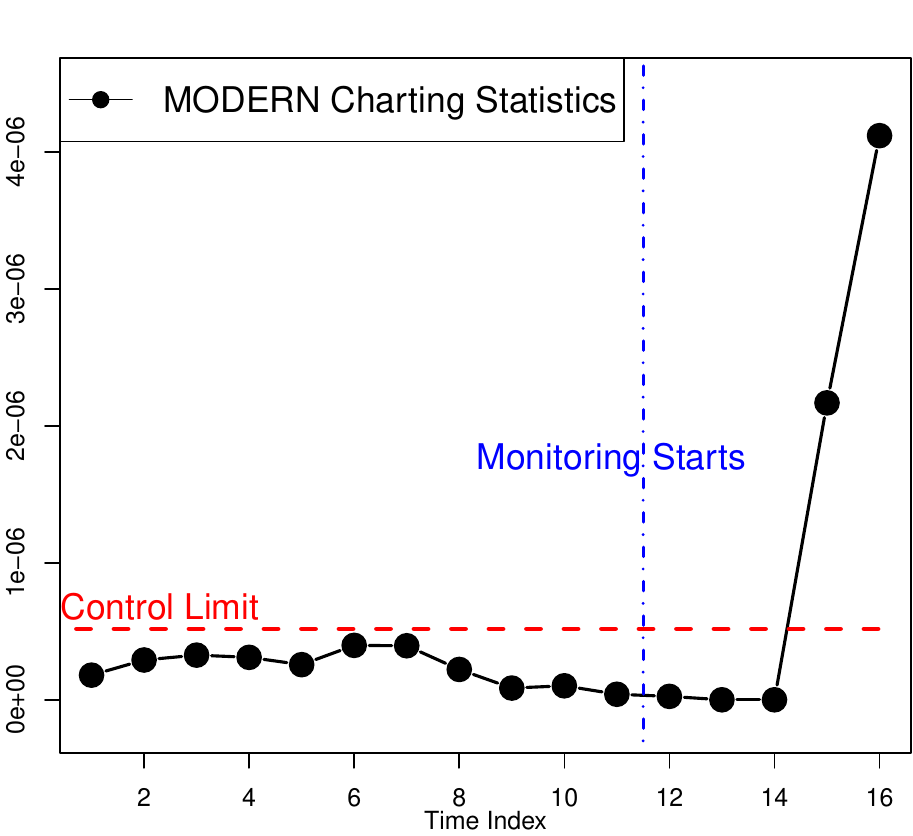}
			 \includegraphics[width = .32\linewidth]{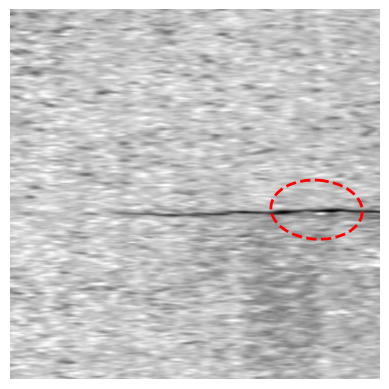}
             \caption{MODERN framework: charting statistics (left) and the faulty isolation (right).}  
\end{subfigure}
~ 
\begin{subfigure}[t]{\linewidth}
\centering
			\includegraphics[width=.32\linewidth]{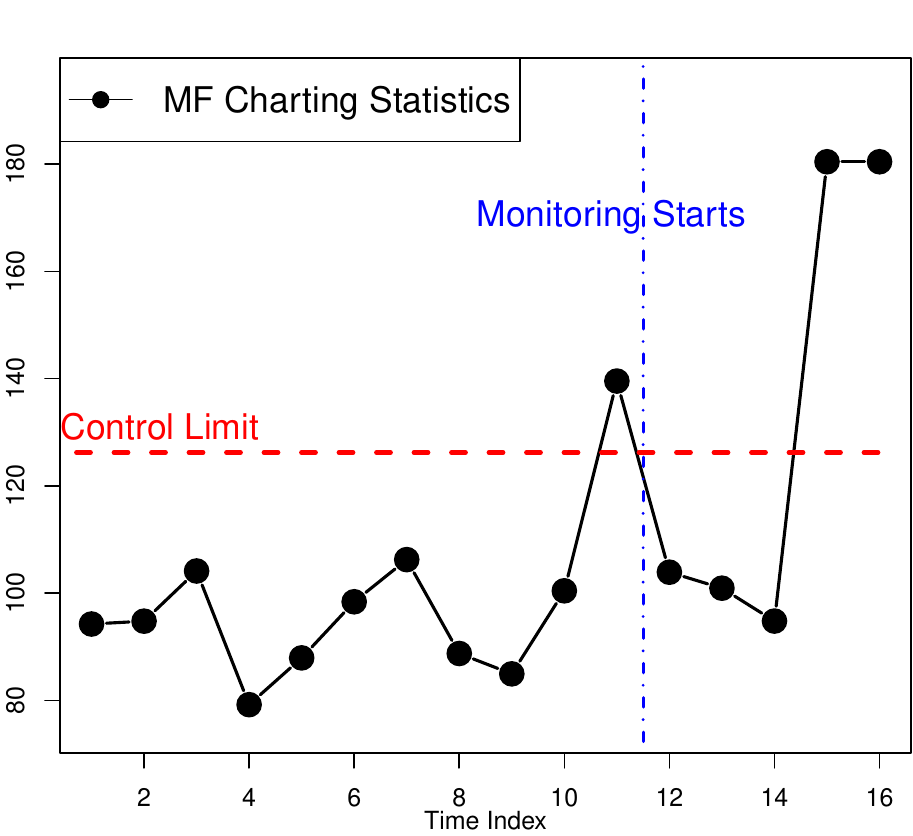} %
            \includegraphics[width = .32\linewidth]{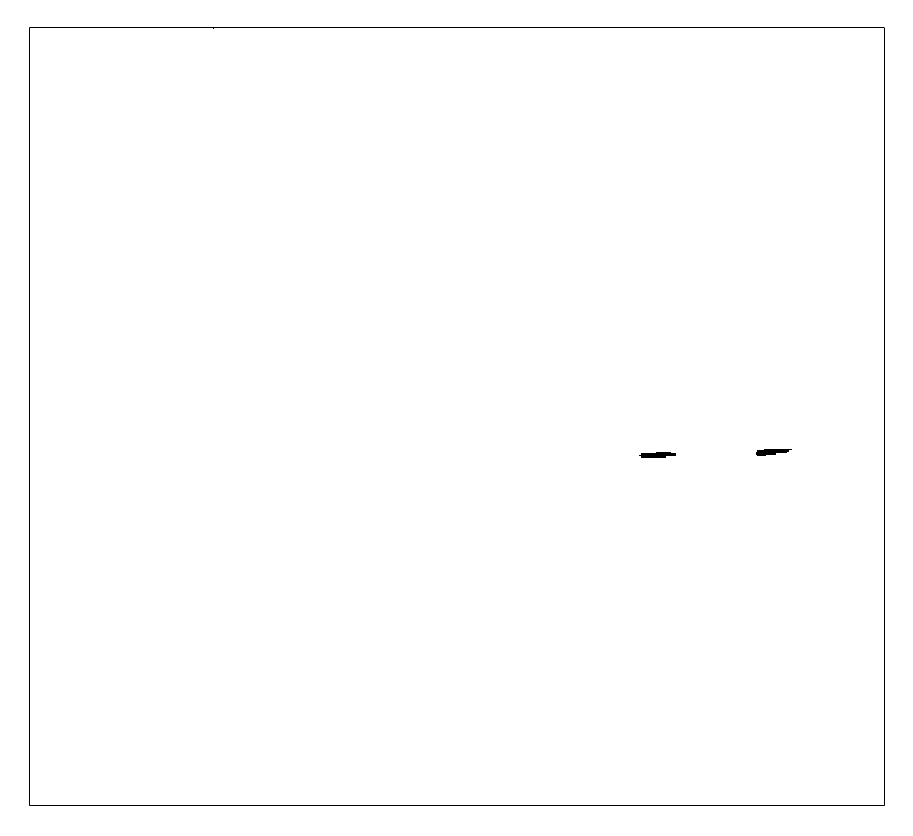}
            \caption{MF method: charting statistics (left) and the faulty isolation (right).}  
\end{subfigure}
\end{figure}
Moreover, we also apply the MF approach to these images for comparison. The corresponding monitoring and fault isolation results are shown in Figure \ref{fig:steel}(b). For the control chart, we observe that, although it gives the signal at the change point, it also has a \emph{false} signal before the process becomes OC (the point above the control limit before the monitoring starts). As for the fault isolation, the MF approach is able to flag some of the defective pixels, a comparable diagnosis performance to our network $\mathscr{G}(\cdot; \widehat{\mathcal{U}})$. 

{\color{black}
\subsection{Application to Leather Manufacturing}
\label{sec:5.5}

In this subsection, we apply our method to some real images in leather manufacturing. The leather images were obtained from the MVTec Anomaly Detection Database \citep{bergmann2021mvtec}, which is made available by MVTec Software GmbH  (\url{https://www.mvtec.com/}), a leading international manufacturer of machine vision software. There are 277 IC images and 19 OC images with cut defects. Our applicability test statistic has a value of $16.204$ and a p-value $< 0.001$. Therefore, we deem our pre-trained MODERN $\mathscr{F}(\cdot; \widehat{\mathcal{V}})$ applicable in this setting. Next, we used 245 (out of 277) IC images to determine the control limit to achieve $\mathrm{ARL}_0 = 50$. The remaining 32 IC images together with the OC images are reserved for the monitoring phase. As monitoring begins, the first 10 images are randomly drawn from the 32 IC images and the subsequent images are randomly drawn from the OC images. In Figure \ref{fig:mvtec}, with the same setup as in Figure \ref{fig:steel}, we compare our results with those given by the MF approach. 

\begin{figure}[ht!]
\centering
\caption{Comparison between MODERN framework and MF method in leather manufacturing.} \label{fig:mvtec}
\begin{subfigure}[t]{\linewidth}
\centering
			 \includegraphics[width = .32\linewidth, height = 0.30\linewidth]{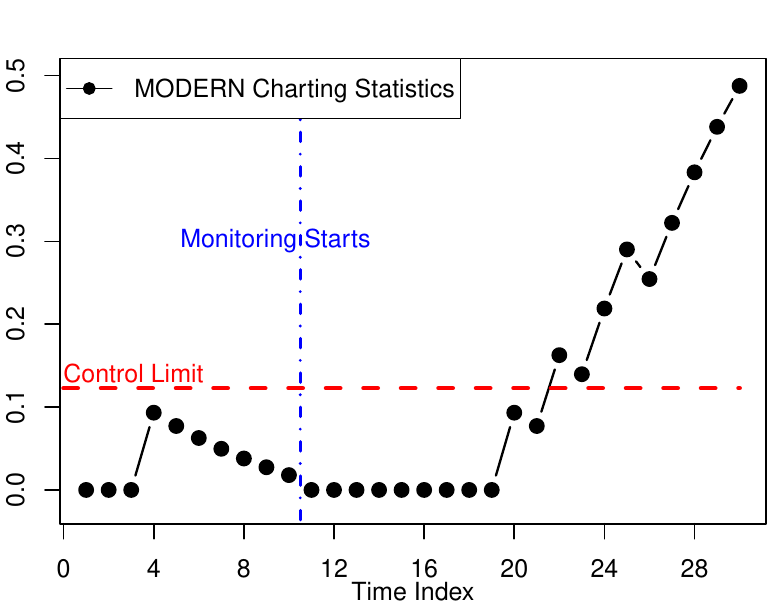}
			 \includegraphics[width = .32\linewidth, height = 0.30\linewidth]{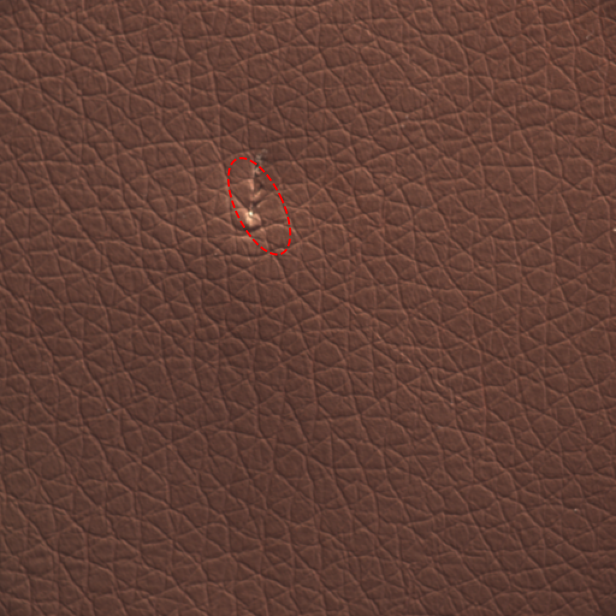}
             \caption{MODERN framework: charting statistics (left) and the faulty isolation (right).}  
\end{subfigure}
~ 
\begin{subfigure}[t]{\linewidth}
\centering
			\includegraphics[width=.32\linewidth]{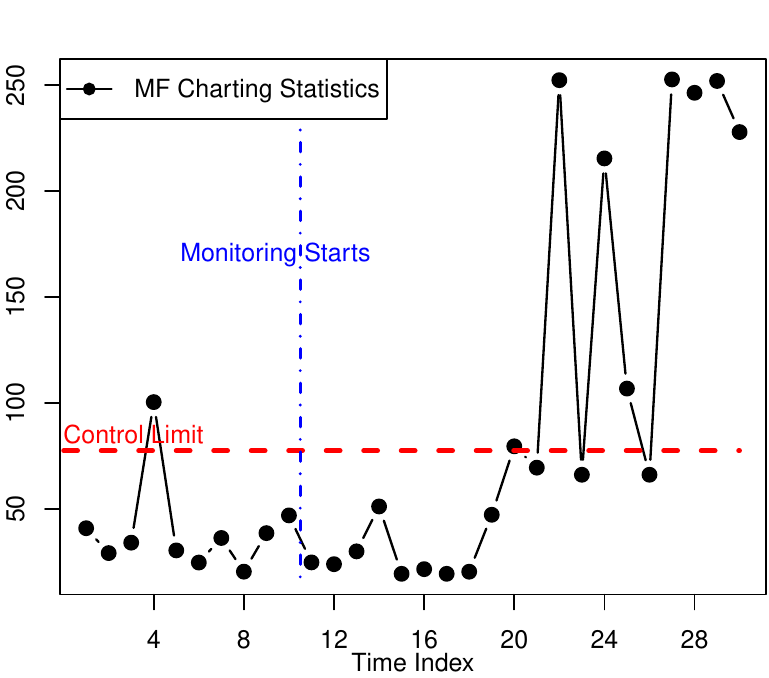} %
            \includegraphics[width = .32\linewidth]{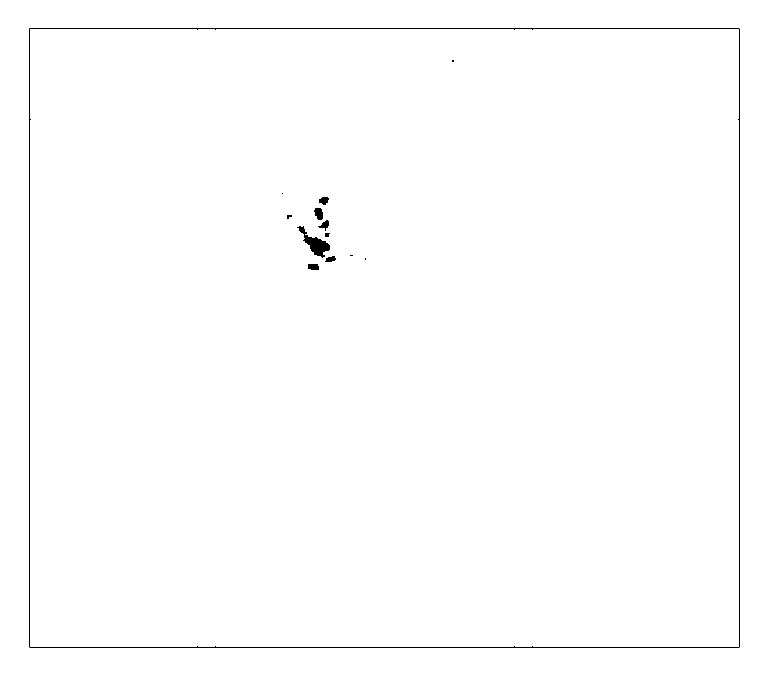}
            \caption{MF method: charting statistics (left) and the faulty isolation (right).}  
\end{subfigure}
\end{figure}

In the control charts, we have included 10 IC images prior to the monitoring phase. Therefore, the process is OC at $t = 21$ and onward. It can be seen that both methods provide a warning at $t = 22$. However, the MF method has two false signals while the process is still IC. Additionally, after the process has shifted, the MF chart returned to the IC state a few times, whereas our MODERN chart correctly suggests a persistent shift. This is because the EWMA mechanism in our MODERN chart is good at accumulating OC evidence overtime. As for the fault isolation, both methods have labeled the defective region reasonably well.} 


\section{Concluding Remarks and Possible Generalizations}
\label{sec:6}

We have proposed a deep learning-based framework (MODERN) for quality monitoring and fault diagnosis in a smart manufacturing system. By utilizing the powerful architecture of inception residual neural networks, our charting method achieves fast signaling when the process becomes abnormal, and our faulty region estimator identifies the defective area well in post-signal diagnosis. To extend our method to cases where there are not enough image data to sufficiently train the networks, we also have proposed a transfer monitoring technique that only requires a small sample size and suggested a hypothesis testing approach for assessing the method's applicability. Theoretically, we have established the optimal convergence rate for both our defect likelihood estimation and fault isolation. This result leads to an interesting and important implication - upgrading monitoring equipment (even for free) may not help a manufacturer's quality control practice. The decision whether to upgrade should instead be based on the fidelity level of image labels. This applies to both monitoring and diagnosis. In numerical studies, comparisons with a state-of-the-art approach show that the proposed method performs well. Finally, we have demonstrated our method with real applications to electric commutator and leather manufacturing. {\color{black}As an important next step in advancing this research, we plan to conduct randomized controlled experiments with industry partners to evaluate the efficacy of MODERN framework in real settings and to gain additional operational insights.}

The proposed quality control framework is not limited to industrial manufacturing. In the food industry, for instance, advanced optics are now available for visualizing both the physical appearance (e.g., shape and size) and the nutritional parameters (e.g., fat, sugar and fibers) of a food product. Our method can be integrated into the production process for automated food quality monitoring and inspection. In the farming business, images from drones can provide farmers with real-time information about when the fruits are ready to be picked. In this setting, our framework can be used for crop surveillance in a smart farming facility. In environmental monitoring,  the Landsat project of NASA (\url{https://landsat.gsfc.nasa.gov/}) has launched 8 satellites to monitor the Earth surface. The image sequences can be used for monitoring changes in forestation, water resource availability, and urban landscapes. We anticipate that our method can be useful in these applications as well.

\begin{flushleft}
\bibliographystyle{ormsv080}
\bibliography{myref}
\end{flushleft}

\newpage
\setcounter{page}{1}
\begin{APPENDICES}


%
%
%


\section{Proofs of Theorems \ref{ester}, \ref{esterlad}, \ref{esternoisy} and \ref{thm:transfer}}
\begin{lemma}\label{lem1}
Let $({X},Y)$ be a random variable sampled from  $ p({x},y)$  with the Bernoulli  random variable $Y \sim \mathrm{Bernoulli}(\sigma(f^*(X)))$, where $\sigma(t) = 1/(1+e^{-t})$.
 Let  $$D^*(x) = \arg\min\limits_{D}\mathbb{E}_{(X, Y) \sim p({x},y)} -[Y\log (\sigma(D(x))) + (1-Y)\log(1-\sigma(D(x)))].$$
 Then $D^* =f^*$, a.s.
\end{lemma}
\begin{proof}
$D^*({x})$ is the minimizer of
\begin{align*}
& \mathbb{E}_{(X, Y) \sim p({x},y)} -[Y\log (\sigma(D(X))) + (1-Y)\log(1-\sigma(D(X)))]\\
=& \mathbb{E}_{(X, Y) \sim p({x},y)} [Y\log (1+\exp^{-D(X)}) + (1-Y)\log(1+\exp^{D(X)})]\\
=& \mathbb{E}_{X}[\mathbb{E}_{Y|X}[Y\log (1+\exp^{-D(X)}) + (1-Y)\log(1+\exp^{D(X)})]]\\
=& \mathbb{E}_{X}[\log (1+\exp^{-D(X)})/(1+\exp^{-f^*(X)})+\log(1+\exp^{D(X)})/(1+\exp^{f^*(X)})]
\end{align*}
The above criterion is a convex functional of $D(\cdot)$. By setting  the first variation to  zero yields
 $D^* = f^*,$ a.s. 
\end{proof}
\subsubsection*{Proof of Theorem \ref{ester}.}
According to Lemma \ref{lem1},  the underlying $f^*$  is the minimizer of the cross entropy loss at the population level
\begin{equation}\label{lsp}
\mathcal{L}(D) =  \mathbb{E}_{X}[\mathbb{E}_{Y|X}[Y\log (1+\exp^{-D(X)}) + (1-Y)\log(1+\exp^{D(X)})]].
\end{equation}
Then we can estimate $f^*$
non-parametrically with deep neural networks.
Let $\{{Z}_i\}_{i=1}^N= \{({X}_i,Y_i)\}_{i=1}^N$  be samples from   $p({x},y)$.
Let $\mathcal{N}_{\mathcal{D}, \mathcal{W}, \mathcal{B}}$ be the set of ReLU  neural networks  $D_{\phi}$ with parameter $\phi$,  depth   $\mathcal{D}$,   width $\mathcal{W}$, and $ \|D_{\phi}\|_{\infty} \leq \mathcal{B}.$
Define the estimator as the minimizer of  the following cross entropy loss at the sample level
\begin{align}\label{lr}
\widehat{D}_{\phi}({x}) &=\arg\min\limits_{D_{\phi} \in \mathcal{N}_{\mathcal{D}, \mathcal{W},  \mathcal{B}}} \widehat{\mathcal{L}}(D_{\phi})\nonumber\\
& = \frac{1}{N} \sum_{i=1}^N Y_i\log (1+\exp^{-D_{\phi}(X_i)}) + (1-Y_i)\log(1+\exp^{D_{\phi}(X_i)}).
\end{align}
Next we bound the nonparametric estimation error  $\|D({x}) - \widehat{D}_{\phi}({x})\|_{L^2(\nu)}$ where $\nu$ is the marginal distribution of $X$.

First we show  the  landscape of the restriction of $\mathcal{L}$ on   $\mathcal{N}_{\mathcal{D}, \mathcal{W}, \mathcal{B}}$ behaves like  quadratic functional around $D^{*}$.
\begin{equation}\label{A6}
C_{1,\mathcal{B}}\|{D} - D^{*}\|_{L^2(\nu)}^2 \leq \mathcal{L}({D})-\mathcal{L}(D^{*})\leq C_{2,\mathcal{B}}\|{D} - D^{*}\|_{L^2(\nu)}^2, \ \ \forall {D} \in \mathcal{N}_{\mathcal{D}, \mathcal{W}, \mathcal{B}},
\end{equation}
where $C_{1,\mathcal{B}} = \frac{\exp^{-\mathcal{B}}}{(1+\exp^{\mathcal{B}})^3},$ $C_{2,\mathcal{B}}=\exp^{\mathcal{B}}.$
We use $C_1,...,C_5$ to denote universal numerical constants in the following.\\
 \textbf{Proof of \eqref{A6}}.\\
 By the proof of Lemma \ref{lem1}, we have $\forall D \in \mathcal{N}_{\mathcal{D}, \mathcal{W}, \mathcal{B}}$
 $$\mathcal{L}(D)-\mathcal{L}(D^{*}) =\mathbb{E}_{X}[\log \frac{1+\exp^{-D(X)}}{1+\exp^{-D^*(X)}}/(1+\exp^{-D^*(X)})+\log\frac{1+\exp^{D(X)}}{1+\exp^{D^*(X)}}/(1+\exp^{D^*(X)})].$$
 $\forall a \in [-\mathcal{B},\mathcal{B}]$, define  $$g(b) = (1+\exp^{-a})^{-1}\log(\frac{1+\exp^{-b}}{1+\exp^{-a}})+(1+\exp^{a})^{-1}\log(\frac{1+\exp^{b}}{1+\exp^{a}}) :b\in  [-\mathcal{B},\mathcal{B}]\rightarrow \mathbb{R}.$$
 The fact $g(a) = g^{\prime}(a) = 0$ and  Taylor expansion implies,
 $$g(b) = \frac{g^{\prime\prime}(\xi)}{2}(b-a)^2,$$ where $\xi \in [-\mathcal{B},\mathcal{B}].$
  The desired results in \eqref{A6} follow from  $$\frac{2\exp^{-\mathcal{B}}}{(1+\exp^{\mathcal{B}})^3}\leq g^{\prime\prime}(\xi) = (1+\exp^{-a})^{-1}\frac{\exp^{-\xi}}{(1+\exp^{-\xi})^2}+ (1+\exp^{a})^{-1}\frac{\exp^{\xi}}{(1+\exp^{\xi})^2}\leq 2\exp^{\mathcal{B}}.$$
  It follow from  \eqref{A6}  that $\forall \bar{D}_{\phi} \in \mathcal{N}_{\mathcal{D}, \mathcal{W}, \mathcal{B}}$,
  \begin{align}
&C_{1,\mathcal{B}}\|\widehat{D}_{\phi} - D^{*}\|_{L^2(\nu)}^2 \leq \mathcal{L}(\widehat{D}_{\phi})-\mathcal{L}(D^{*}) \nonumber\\
 & =\mathcal{L}(\widehat{D}_{\phi}) - \widehat{\mathcal{L}}(\widehat{D}_{\phi}) +    \widehat{\mathcal{L}}(\widehat{D}_{\phi})-  \widehat{\mathcal{L}}(\bar{D}_{\phi}) \nonumber +   \widehat{\mathcal{L}}(\bar{D}_{\phi}) - \mathcal{L}(\bar{D}_{\phi}) +\mathcal{L}(\bar{D}_{\phi})  - \mathcal{L}(D^*)\nonumber\\
  &\leq  2 \sup_{D \in \mathcal{N}_{\mathcal{D}, \mathcal{W},  \mathcal{B}}} |\mathcal{L}(D) - \widehat{\mathcal{L}}(D) |+ C_{2,\mathcal{B}}\inf_{D\in  \mathcal{N}_{\mathcal{D}, \mathcal{W},  \mathcal{B}}}\|D - D^{*}\|_{L^2(\nu)}^2, \label{A7}
  \end{align}
  where we use the definition of   $\widehat{D}_{\phi}$ and $\bar{D}_{\phi}$.
   Let ${Z} = ({X},Y)$ be the  random variable pair drawn from $p({x},y)$.
  Define  $$b(D,{Z}) = Y\log (1+\exp^{-D(X)}) + (1-Y)\log(1+\exp^{D(X)}).$$
  It is easy to check that $b(D,{Z})$ is 1-Lipschitz on $D$, i.e.,
  \begin{equation}\label{lip}
|b(D,{Z})-b(\tilde{D},{Z})|\leq |D({x}) - \tilde{D}({x})|.
  \end{equation}
  Let $\widetilde{{Z}}_i$ be a i.i.d. copy of ${Z}_{i},$ and $\sigma_i (\epsilon_i) $ be the i.i.d. Rademacher random (standard  normal) variables that are independent with
  $\widetilde{{Z}}_i$ and ${Z}_{i}$,  $i = 1,...N.$
  We need the following
  results \eqref{A8}-\eqref{A11} to upper bounding the expected value of the right hand side term in \eqref{A7}.
 \begin{equation}\label{A8}
 \mathbb{E}_{\{{O}_i\}_{i=1}^N} [\sup_{D} |\mathcal{L}(D) - \widehat{\mathcal{L}}(D) | ] \leq 4C_1 \mathcal{G}(\mathcal{N}),
 \end{equation}
  where $\mathcal{G}(\mathcal{N})$ is the Gaussian complexity (Bartlett and Mendelson 2002) of $\mathcal{N}_{\mathcal{D}, \mathcal{W}, \mathcal{B}}$ defined as
  $$\mathcal{G}(\mathcal{N}) =  \mathbb{E}_{\{{Z}_i, \epsilon_i \}_{i}^N}[\sup_{D\in \mathcal{N}_{\mathcal{D}, \mathcal{W}, \mathcal{B}}}|\frac{1}{N}\sum_{i=1}^N\epsilon_i b(D,{Z}_i)|].$$
  \textbf{Proof of} \eqref{A8}.\\
  Clearly,
   $$\mathcal{L}(D) = \mathbb{E}_{{Z}} [b(D,{Z})] = \frac{1}{N}\mathbb{E}_{\widetilde{{Z}}_i} [b(D,\widetilde{{Z}}_i],$$ and
   $$ \widehat{\mathcal{L}}(D) = \frac{1}{N}\sum_{i=1}^N b(D,{Z}_i).$$
      Let $$\mathcal{R}(\mathcal{N}) = \frac{1}{N} \mathbb{E}_{\{{Z}_i, \sigma_i\}_{i}^N}[\sup_{D\in \mathcal{N}_{\mathcal{D}, \mathcal{W}, \mathcal{B}}}|\sum_{i=1}^N\sigma_i b(D,{Z}_i)|]$$ be the 	
Rademacher complexity of $\mathcal{N}_{\mathcal{D}, \mathcal{W}, \mathcal{B}}$ (Bartlett and Mendelson 2002).
    Then,
  \begin{align*}
  &\mathbb{E}_{\{{Z}_i\}_{i=1}^N} [\sup_{D} |\mathcal{L}(D) - \widehat{\mathcal{L}}(D) | ] \\
  &=\frac{1}{N} \mathbb{E}_{\{{Z}_i\}_{i}^N} [\sup_{D} |\sum_{i=1}^N (\mathbb{E}_{\widetilde{{Z}}_i} [b(D,\widetilde{{Z}}_i)] - b(D,{Z}_i))|]\\
  & \leq  \frac{1}{N} \mathbb{E}_{\{{Z}_i, \widetilde{{Z}}_i\}_{i}^N} [\sup_{D} |b(D,\widetilde{{Z}}_i) - b(D,{{Z}}_i)|]\\
  & =  \frac{1}{N} \mathbb{E}_{\{{Z}_i, \widetilde{{Z}}_i,\sigma_i \}_{i}^N} [\sup_{D } |\sum_{i=1}^N\sigma_i(b(D,\widetilde{{Z}}_i) - b(D,{{Z}}_i))|]\\
  & \leq \frac{1}{N}  \mathbb{E}_{\{{Z}_i, \sigma_i \}_{i}^N} [\sup_{D } |\sum_{i=1}^N\sigma_i b(D,{{Z}}_i)| ] + \frac{1}{N}  \mathbb{E}_{\{\widetilde{{Z}}_i,\sigma_i \}_{i}^N} [\sup_{D } |\sum_{i=1}^N\sigma_i b(D,\widetilde{{Z}}_i)| ] \\
  &= 2\mathcal{R}(b\circ\mathcal{N})\\
  &\leq 4\mathcal{R}(\mathcal{N})\\
  & \leq 4C_1 \mathcal{G}(\mathcal{N}),
  \end{align*}
  where, the first inequality follows from the Jensen's inequality, and the second equality holds since  both $\sigma_i(b(D,\widetilde{{Z}}_i) - b(D,{{Z}}_i))$ and $b(D,\widetilde{{Z}}_i) - b(D,{{Z}}_i)$ are governed by the same law, and the last  equality holds since the distribution of the two terms are the same, and in the  third inequality we use the   Lipschitz contraction property of Rademacher complexity, see Theorem 12 in Bartlett and Mendelson (2002),  and \eqref{lip}, and  the last inequality holds since the  relationship between the Gaussian complexity and  the Rademacher complexity, see for Lemma 4 in Bartlett and Mendelson (2002). 
  Next, we bound the Gaussian complexity.
  \begin{equation}\label{A9}
  \mathcal{G}(\mathcal{N}) \leq
  C_2\mathcal{B} \sqrt{\frac{N}{(\mathcal{D}\mathcal{W})^2\log \mathcal{\mathcal{D}\mathcal{W}}}}\log \frac{N}{(\mathcal{D}\mathcal{W})^2\log \mathcal{\mathcal{D}\mathcal{W}}} \exp^{-\log^2 \frac{N}{(\mathcal{D}\mathcal{W})^2\log \mathcal{\mathcal{D}\mathcal{W}}}}.
  \end{equation}
  \textbf{Proof of \eqref{A9}}.\\
  Since $\mathcal{N}$ is  closed under negation,
\begin{align*}
&\mathcal{G}(\mathcal{N}) =  \mathbb{E}_{\{{Z}_i, \epsilon_i \}_{i}^N}[\sup_{D\in \mathcal{N}_{\mathcal{D}, \mathcal{W}, \mathcal{B}}}\frac{1}{n}\sum_{i=1}^N\epsilon_i b(D,{{Z}}_i)]\\
& = \mathbb{E}_{{Z}_i}[ \mathbb{E}_{\epsilon_i}[\sup_{D\in \mathcal{N}_{\mathcal{D}, \mathcal{W}, \mathcal{B}}}\frac{1}{N}\sum_{i=1}^N\epsilon_i b(D,{{Z}}_i)]|\{{Z}_i\}_{i=1}^N].
\end{align*}
Conditioning on $\{{Z}_i\}_{i =1}^N$,
$\forall D, \widetilde{D} \in \mathcal{N}_{\mathcal{D}, \mathcal{W}, \mathcal{B}}$ it easy to check $$\mathbb{V}_{\epsilon_i} [\frac{1}{N}\sum_{i=1}^n\epsilon_i (b(D,{{Z}}_i) - b(\widetilde{D},{{Z}}_i))] = \frac{d_{\mathcal{N},2}(D,\widetilde{D})}{\sqrt{N}},$$
where, $$d_{\mathcal{N},2}(D,\widetilde{D}) = \frac{1}{\sqrt{N}} \sqrt{\sum_{i =1}^N (b(D,{{Z}}_i) - b(\widetilde{D},{{Z}}_i))^2}.$$
Denote  $\mathfrak{C}(\mathcal{N},d_{\mathcal{N}},\delta)$ as the covering number of $\mathcal{N}_{\mathcal{D}, \mathcal{W},  \mathcal{B}}$
under the metric $d_{\mathcal{N},2}$ with radius $\delta$, and let $\mathrm{Pdim}_{\mathcal{N}}$ be the Pseudo-dimension of $\mathcal{N}_{\mathcal{D}, \mathcal{W}, \mathcal{B}}$.
Since the diameter of $\mathcal{N}_{\mathcal{D}, \mathcal{W},  \mathcal{B}}$ under $d_{\mathcal{N},2} $ is at most $\mathcal{B}$,  we have
  \begin{align*}
  \mathcal{G}(\mathcal{N}) &\leq \frac{C_3}{\sqrt{N}} \mathbb{E}_{\{{Z}_i\}_{i=1}^N}[\int_{0}^{\mathcal{B}} \sqrt{\log \mathfrak{C}(\mathcal{N},d_{\mathcal{N},2},\delta)} \mathrm{d} \delta]\\
  &\leq\frac{C_3}{\sqrt{N}} \mathbb{E}_{\{{Z}_i\}_{i=1}^N}[\int_{0}^{\mathcal{B}} \sqrt{\log  \mathfrak{C}(\mathcal{N},d_{\mathcal{N},\infty},\delta)} \mathrm{d} \delta]\\
  &\leq \frac{C_3}{\sqrt{N}}\int_{0}^{\mathcal{B}} \sqrt{\mathrm{Pdim}_{\mathcal{N}} \log \frac{2e\mathcal{B}N}{\delta \mathrm{Pdim}_{\mathcal{N}} }} \mathrm{d} \delta, \\
  & \leq C_4 \mathcal{B}(\frac{N}{\mathrm{Pdim}_{\mathcal{N}}})^{1/2}\log (\frac{N}{\mathrm{Pdim}_{\mathcal{N}}}) \exp^{ -\log^2(\frac{N}{\mathrm{Pdim}_{\mathcal{N}}})}\\
  & \leq C_2\mathcal{B} \sqrt{\frac{N}{(\mathcal{D}\mathcal{W})^2\log \mathcal{\mathcal{D}\mathcal{W}}}}\log \frac{N}{(\mathcal{D}\mathcal{W})^2\log \mathcal{\mathcal{D}\mathcal{W}}} \exp^{-\log^2 \frac{N}{(\mathcal{D}\mathcal{W})^2\log \mathcal{\mathcal{D}\mathcal{W}}}}
  \end{align*} 
  where, the first inequality follows from the chaining  Theorem 8.1.3 in Vershynin (2018), and the second inequality holds due to
  $\mathfrak{C}(\mathcal{N},d_{\mathcal{N},2},\delta)\leq \mathfrak{C}(\mathcal{N},d_{\mathcal{N},\infty},\delta)$, and in the third inequality we used
  the relationship between the matric entropy and the Pseudo-dimension of the ReLU networks  $\mathcal{N}_{\mathcal{D}, \mathcal{W},  \mathcal{B}}$ (Anthony and Bartlett 1999), i.e., 
  $$\log \mathfrak{C}(\mathcal{N},d_{\mathcal{N},\infty},\delta)) \leq \mathrm{Pdim}_{\mathcal{N}} \log \frac{2e\mathcal{B}N}{\delta\mathrm{Pdim}_{\mathcal{N}}},$$
  and the fourth inequality follows by  some calculation,
  and the last inequality  holds due to the  upper bound of Pseudo-dimension for the ReLU network $\mathcal{N}_{\mathcal{D}, \mathcal{W},  \mathcal{B}}$ satisfying  $$\mathrm{Pdim}_{\mathcal{N}} \leq C_5 (\mathcal{D}\mathcal{W})^2\log \mathcal{\mathcal{D}\mathcal{W}},$$ see Bartlett et al. (2019). 
  Last,  we bound the approximation error $\inf_{D\in \mathcal{N}_{\mathcal{W},\mathcal{D},\mathcal{B}}}\|D - D^{*}\|_{L^2(\nu)}$.

  \begin{lemma}\label{shen}
For any function $D^*:[-B,B]^d\rightarrow \mathbb{R}$ with Lipschitz constant $L$
there exist a ReLU network $\bar{D}$ with depth $\mathcal{O}(12\mathcal{D}+C_{1,d}) $ and width $\mathcal{O}(C_{2,d} \mathcal{W})$
such that
\begin{equation}\label{A11}
\|D^*-\bar{D}\|_{L^{\infty}} \leq 19L\sqrt{d} B (\mathcal{D}\mathcal{W})^{-2/d},
\end{equation}
where $C_{1,d} = 14+2d$, $C_{2,d} = 3^{d+3}.$
\end{lemma}
\begin{proof}
This Lemma follows directly from  Theorem 1.1 of Shen et al. (2019). 
\end{proof}

Now, combing the results  $\eqref{A7} - \eqref{A11}$,
 we have
 \begin{align*}
 & \mathbb{E}_{\{{Z}_{i}\}_{i=1}^N}\|\widehat{D}_{\phi} - D^{*}\|_{L^2(\nu)}^2\\
&\leq C_{d,\mathcal{B}}(\sqrt{\frac{N}{(\mathcal{D}\mathcal{W})^2\log \mathcal{\mathcal{D}\mathcal{W}}}}\log \frac{N}{(\mathcal{D}\mathcal{W})^2\log \mathcal{\mathcal{D}\mathcal{W}}} \exp^{-\log^2 \frac{N}{(\mathcal{D}\mathcal{W})^2\log \mathcal{\mathcal{D}\mathcal{W}}}}+     (19L\sqrt{d} B (\mathcal{D}\mathcal{W})^{-2/d})^2)\\
&\leq C_{d,\mathcal{B},B,L} (\sqrt{\frac{(\mathcal{\mathcal{D}\mathcal{W}})^2\log(\mathcal{W}\mathcal{D})}{N}}\log N+(\mathcal{\mathcal{D}\mathcal{W}})^{-4/d})\\
& \leq  C_{d,\mathcal{B},B,L}N^{-2/(2+ d)}\log^{3/2}N,
 \end{align*}
 where, the last  inequality holds since we set  $\mathcal{D}\mathcal{W} = \mathcal{O}(N^{d/(2d+4)})$  and some algebra.

\subsubsection*{Proof of Theorem \ref{esterlad}.}
 With $L_1$ loss, it is easy to check that 
\begin{equation}\label{lossp}
g^*\in \arg \min_{\mathrm{measurable}  \ \ g} \mathcal{L}(g) = \mathbb{E}_{X,z} [w^\prime|g(X)-z|].
\end{equation}
To estimate $g^*$, we consider  the following deep estimator 
  $$\hat{g} \in \arg\min_{g\in \mathcal{N}_{\mathcal{W}, \mathcal{D},\mathcal{B}}} \widehat{L}(g)= \sum_{j=1}^n \omega^\prime |g(X_j)-z_j|,$$
where $X_j, j=1,..., n$ are OC images and $z_j \in \mathbb{R}^5$ denote the elliptic parameter corresponding to the defect domain,
$\omega $ is the weight (we set  $ \omega = 1/5$ for ease of presentation.)

We first in Lemma (\ref{errdec}) show that  the prediction error $\mathcal{L}(\hat{g})-\mathcal{L}(g^*)$ can be bounded  by a  empirical process (variance) that indexed by network class $\mathcal{N}_{\mathcal{W},\mathcal{D},\mathcal{B}}$ plus the approximation  error (bias)
\begin{lemma}\label{errdec}
$$
\mathbb{E}[\mathcal{L}(\hat{g})-\mathcal{L}(g^*)]\leq \mathbb{E}[ 2\sup_{g \in \mathcal{N}_{\mathcal{D}, \mathcal{W},  \mathcal{B}}} |\mathcal{L}(g) - \widehat{\mathcal{L}}(g) |+ \inf_{g\in  \mathcal{N}_{\mathcal{D}, \mathcal{W},  \mathcal{B}}}\|g - g^{*}\|_{L^1(\nu)}].$$
\end{lemma}
\begin{proof}
$\forall \bar{g} \in \mathcal{N}_{\mathcal{D}, \mathcal{W},  \mathcal{B}}$
\begin{align}
& \mathcal{L}(\hat{g})-\mathcal{L}(g^{*}) \nonumber\\
 & =\mathcal{L}(\hat{g}) - \widehat{\mathcal{L}}(\hat{g}) +    \widehat{\mathcal{L}}(\hat{g})-  \widehat{\mathcal{L}}(\bar{g}) \nonumber +   \widehat{\mathcal{L}}(\bar{g}) - \mathcal{L}(\bar{g}) +\mathcal{L}(\bar{g})  - \mathcal{L}(g^*)\nonumber\\
  &\leq  2 \sup_{g \in \mathcal{N}_{\mathcal{D}, \mathcal{W},  \mathcal{B}}} |\mathcal{L}(g) - \widehat{\mathcal{L}}(g) |+ \inf_{D\in  \mathcal{N}_{ \mathcal{W},\mathcal{D},  \mathcal{B}}}\|g - g^{*}\|_{L^1(\nu)}, \label{A17}
  \end{align}
  where we use the definition of   $\hat{g}_{\phi}$ and $\bar{g}$ and 1-Lipschitz property of $\mathcal{L}(g)$ in (\ref{lossp}).
\end{proof}
Next, we bound the two terms in right hand side of Lemma \ref{errdec} separately in the following two Lemmas, respectively.
\begin{lemma}\label{empp}
$$
\mathbb{E}[ 2\sup_{g \in \mathcal{N}_{\mathcal{D}, \mathcal{W},  \mathcal{B}}} |\mathcal{L}(g) - \widehat{\mathcal{L}}(g) |]\leq  C_2\mathcal{B} \sqrt{\frac{N}{(\mathcal{D}\mathcal{W})^2\log \mathcal{\mathcal{D}\mathcal{W}}}}\log \frac{N}{(\mathcal{D}\mathcal{W})^2\log \mathcal{\mathcal{D}\mathcal{W}}} \exp^{-\log^2 \frac{N}{(\mathcal{D}\mathcal{W})^2\log \mathcal{\mathcal{D}\mathcal{W}}}}.$$
\end{lemma}
\begin{proof}
The same proof process as for deriving (\ref{A8})-(\ref{A9})
\end{proof} 

\begin{lemma}\label{applad}
For any function $g^*:[-B,B]^d\rightarrow \mathbb{R}$ with Lipschitz constant $L$
there exist a ReLU network $\bar{g}$ with depth $\mathcal{O}(12\mathcal{D}+C_{1,d}) $ and width $\mathcal{O}(C_{2,d} \mathcal{W})$
such that
\begin{equation*}
\inf_{{g} \in \mathcal{N}_{\mathcal{W},\mathcal{D},\mathcal{B}}} \|g^*-{g}\|_{L^{1}(\mu)} \leq \|g^*-\bar{g}\|_{L^{\infty}} \leq 19L\sqrt{d} B (\mathcal{D}\mathcal{W})^{-2/d},
\end{equation*}
where $C_{1,d} = 14+2d$, $C_{2,d} = 3^{d+3}.$
\end{lemma}
\begin{proof}
Since $
\inf_{{g} \in \mathcal{N}_{\mathcal{W},\mathcal{D},\mathcal{B}}} \|g^*-{g}\|_{L^{1}(\mu)}\leq 
\inf_{{g} \in \mathcal{N}_{\mathcal{W},\mathcal{D},\mathcal{B}}} \|g^*-{g}\|_{L^{\infty}}
$, Lemma \ref{applad} fellows from Lemma \ref{shen}.
\end{proof}

Now, combing the results in Lemma \ref{errdec}-Lemma \ref{applad},
 we have
 \begin{align*}
 & \mathbb{E}[\mathcal{L}(\hat{g})-\mathcal{L}(g^*)]\\
&\leq C_{d,\mathcal{B}}(\sqrt{\frac{N}{(\mathcal{D}\mathcal{W})^2\log \mathcal{\mathcal{D}\mathcal{W}}}}\log \frac{N}{(\mathcal{D}\mathcal{W})^2\log \mathcal{\mathcal{D}\mathcal{W}}} \exp^{-\log^2 \frac{N}{(\mathcal{D}\mathcal{W})^2\log \mathcal{\mathcal{D}\mathcal{W}}}}+     19L\sqrt{d} B (\mathcal{D}\mathcal{W})^{-2/d})\\
&\leq C_{d,\mathcal{B},B,L} (\sqrt{\frac{(\mathcal{\mathcal{D}\mathcal{W}})^2\log(\mathcal{W}\mathcal{D})}{N}}\log N+(\mathcal{\mathcal{D}\mathcal{W}})^{-2/d})\\
& \leq  C_{d,\mathcal{B},B,L}N^{-1/(2+ d)}\log^{3/2}N,
 \end{align*}
 where, the last  inequality holds since we set  $\mathcal{D}\mathcal{W} = \mathcal{O}(N^{d/(2d+4)})$  and some algebra.

 Next, we show part (ii) of the theorem. Since we set $\mathcal{B} = 2\|g^*\|_{L^{\infty}}$, then
$\|\hat{g}-g^*\|_{L_2(\nu)}^2
\leq \min \{ \|\hat{g}-g^*\|_{L_2(\nu)}^2, 1.5\mathcal{B}\|\hat{g}-g^*\|_{L_1(\nu)}\}$.
By the above display and the additional calibration assumption, we can show (similar as that in \citealt{belloni2011L1, madrid2022risk})
\begin{equation}\label{calb}
\|\hat{g}-g^*\|_{L^2(\nu)}^2\leq C_{c_1,c_2,\mathcal{B}} (\mathcal{L}(\hat{g}) -\mathcal{L}(g^*)),
\end{equation}
where $C_{c_1,c_2,\mathcal{B}} $ is a constant only depends on $c_1,c_2, \mathcal{B}.$ 
This immediately implies that
$$\mathbb{E}_{(X_i,z_i)_{i=1}^N}[\|\hat{g}-g^*\|_{L^2(\nu)}^2] = \mathcal{O}(N^{-1/(2+ d)}\log^{3/2}N). $$

\vspace{10pt}

Lastly, we make a remark concerning the convergence rate established in Theorem \ref{esterlad} for $\mathbb{E}\|\mathscr{G}(\cdot; \widehat{\mathcal{U}}) - g^*\|_{L^2(\nu)}^2$. It is slower than the minimax optimal rate as established in Theorem \ref{ester}. This is due to the technical difficulty brought up by the $L_1$ loss. If the least square loss is adopted and the objective function in the minimization, we can still obtain the same optimal rate according to our above proof. 

{\color{black}
\subsubsection*{Proof of Theorem \ref{esternoisy}.}
We use $(X,\tilde{Y})\sim \tilde{p}(x,y)$ to denote the noisy random pair of $(X,\tilde{Y})$ whose distribution is $p(x,y)$ in Lemma \ref{lem1}. Define 
\begin{equation}\label{lspnoisy}
\tilde{\mathcal{L}}(D) =  \mathbb{E}_{X}[\mathbb{E}_{\tilde{Y}|X}[\tilde{Y}\log (1+\exp^{-D(X)}) + (1-\tilde{Y})\log(1+\exp^{D(X)})]].
\end{equation}  
By some routine calculations, we have  $$
\tilde{Y}=
\left\{
  \begin{array}{ll}
    1, & \mathrm{with \quad probablity} \quad (1-\delta)\sigma(f^{*}(X))+\delta(1-\sigma(f^*{X})); \\
    0, & \mathrm{with \quad probability} \quad  (1-\delta)(1-\sigma(f^{*}(X)))+\delta\sigma(f^*{X}).
  \end{array}
\right.
$$
Plugging the above expression into (\ref{lspnoisy}) and using the definition of 
$\mathcal{L}(D)$ in (\ref{lsp}) we have
$$\tilde{\mathcal{L}}(D)={\mathcal{L}}(D)+\mathbb{E}_{X}[\log(\frac{1+\exp^{-D(X)}}{1+\exp^{D(X)}})\delta(1-2\sigma(f^{*}(X)))].$$
It follows from the boundedness of $D$ that 
\begin{equation}\label{relation}
|\tilde{\mathcal{L}}(D)-{\mathcal{L}}(D)|\leq |\exp^{D(X)}+\exp^{D(X)}||1-2\sigma(f^{*}(X))|\delta \leq 3(1+C_{2,\mathcal{B}})\delta.
\end{equation}
Let \begin{align}\label{lrnoisy}
\tilde{\widehat{D}}_{\phi}({x}) &=\arg\min\limits_{D_{\phi} \in \mathcal{N}_{\mathcal{D}, \mathcal{W},  \mathcal{B}}} \widehat{\mathcal{L}}(D_{\phi})\nonumber\\
& = \frac{1}{N} \sum_{i=1}^N \tilde{Y}_i\log (1+\exp^{-D_{\phi}(X_i)}) + (1-\tilde{Y}_i)\log(1+\exp^{D_{\phi}(X_i)}),
\end{align}
where $\{(X_i,\tilde{Y}_i)\}$ are i.i.d. random pairs with distribution $\tilde{p}(x,y)$.
By the same arguments in Theorem \ref{ester}, we can derive 
$$\tilde{\mathcal{L}}(\tilde{\widehat{D}})-\tilde{\mathcal{L}}(f^*)
=\tilde{\mathcal{L}}(\tilde{\widehat{D}})-\tilde{\mathcal{L}}(D^*)\leq \mathcal{O}(N^{-2/(2+ d)}\log^{3/2}N).$$
Combing the above equation, (\ref{relation}) and (\ref{A6}), we have 
$$C_{1,\mathcal{B}}\|\tilde{\widehat{D}}-D^*\|^{2}_{L^{2}(\nu)}\leq \mathcal{L}(\tilde{\widehat{D}})-\mathcal{L}(D^*)\leq \tilde{\mathcal{L}}(\tilde{\widehat{D}})-\tilde{\mathcal{L}}(D^*)+6(1+C_{2,\mathcal{B}})\delta\leq \mathcal{O}(N^{-2/(2+ d)}\log^{3/2}N)+\mathcal{O}(\delta).$$
}
{\color{black}
\subsubsection*{Proof of Theorem \ref{thm:transfer}.}
Next, we prove Theorem \ref{thm:transfer}. We first introduce some notations. For a given sequence $\bm{x} = \{x_t: t \geq t \}$, denote the corresponding EWMA transformation by $\mathcal{E}(x_t)$. That is,
\begin{align*}
    \mathcal{E}(x_0) = 0, \quad \mathcal{E}(x_t) = \max\left\{0, \lambda x_t + (1 - \lambda) \mathcal{E}(x_{t - 1})\right\}, \; \text{ for } t \geq 1.
\end{align*}
Let $A_t$ denote the subset in $\mathbb{R}^t$ defined by 
\begin{align*}
    A_t = \left\{(x_1, \ldots, x_t) \in \mathbb{R}^t: \mathcal{E}(x_1) \leq \rho\sqrt{\frac{\lambda\sigma^2_{\texttt{IC}}}{2 - \lambda}}, \ldots, \mathcal{E}(x_{t - 1}) \leq \rho\sqrt{\frac{\lambda\sigma^2_{\texttt{IC}}}{2 - \lambda}}, \mathcal{E}(x_t) > \rho\sqrt{\frac{\lambda\sigma^2_{\texttt{IC}}}{2 - \lambda}} \right\}.
\end{align*}
Let $F_0$ denote the cumulative distribution function (CDF) of $\log\left(\mathscr{F}(\bm{X}; \mathcal{V})/(1 - \mathscr{F}(\bm{X}; \mathcal{V}) \right)$ when $\bm{X}$ is IC. Let $\mathcal{B}(F; \epsilon)$ denote the open ball of radius $\epsilon$ with the $L^\infty$ norm.

\noindent\underline{The Run Length Regularity Condition}.
Let $\mathscr{S}$ be a collection of linear combinations of CDFs such that 
\begin{align}
    \label{eq:transfer-condition}
    \left| \int_{A_t} \prod_{j=1}^t d H_j(x_j) \right| = \prod_{j=1}^t \| H_j \|_\infty \mathcal{O}\left( \frac{1}{t^{3 + \xi}}\right), \; t = 1, 2, \ldots,
\end{align}
where $H_j \in \mathscr{S}$ for $j = 1, 2, \ldots$, $\| \cdot \|_\infty$ is the $L^\infty$ norm, and $\xi > 0$ is a constant. 

\begin{lemma}
    \label{lemma:gateaux}
    Let $F$ be a CDF. Assume that $\mathscr{S}$ contains an open neighborhood of $F$ and the origin. Define the functional $\mathscr{M}(F) = \sum_{t = 1}^\infty t \int_{A_t} \prod_{j = 1}^t d F(x_j)$. Then the Gateaux differential of $\mathscr{M}(F)$ is 
    \begin{align*}
        \mathcal{D}_g \mathscr{M}(F; G) = \sum_{t = 1}^\infty t \eta_t(F; G) = \sum_{t = 1}^\infty t \left[ \sum_{k = 1}^t \int_{A_t} \prod_{ j = 1}^{k -1} d F(x_j) d G(x_k) \prod_{j = k + 1}^t d F(x_j) \right],
    \end{align*}
    for any $G \in \mathscr{S}$.
\end{lemma}
\begin{proof}
    Write 
    \begin{align*}
        \mathscr{M}(F + \tau G) &= \sum_{t = 1}^\infty t \int_{A_t} \prod_{j = 1}^t \left[ d F(x_j) + \tau d G(x_j)\right] \\
        & = \sum_{t = 1}^\infty t \int_{A_t} \prod_{j = 1}^t d F(x_j) + \tau \left\{\sum_{t = 1}^\infty t \left[ \int_{A_t} dG(x_1) \cdots dF(x_t) + \cdots + \int_{A_t} dF(x_1) \cdots d G(x_t) \right] \right\} + \cdots,
    \end{align*}
    where $\cdots$ denotes the remainder that involves $\tau^2$. Continuing with the Gateaux differential calculation, we have 
    \begin{align*}
        \lim_{\tau \rightarrow 0+} \frac{\mathscr{M}(F + \tau G) - \mathscr{M}(F)}{\tau} = \sum_{t = 1} ^\infty t \eta_t(F; G).
    \end{align*}
    To see that the summation in $\mathcal{D}_g\mathscr{M}(F; G)$ is finite, write 
    \begin{align*}
        \left| \eta_t(F; G) \right| \leq t \| F \|_\infty^{t - 1} \| G \|_\infty \frac{C}{t^{3 + \xi}} \leq C \| G \|_\infty \frac{1}{t^{2 + \xi}},
    \end{align*}
    where $C$ is a constant. Here we have used the fact that $F$ is a CDF and thus $\|F \|_\infty = 1$. It follows that 
    \begin{align*}
        \sum_{t = 1}^\infty t |\eta_t(F; G) | \leq C \| G \|_\infty\sum_{t = 1}^\infty \frac{1}{t^{1 + \xi}} < \infty. 
    \end{align*}
\end{proof}

It is easy to verify that $\mathrm{ARL}(\rho) = \mathscr{M}(F_0)$. Therefore, $F_0 \in \mathscr{S}$ effectively ensures that the average run length is finite and the second moment of the run length also exists. We are now ready to prove Theorem \ref{thm:transfer}.

\subsubsection*{Proof of Theorem \ref{thm:transfer}}. The bootstrap algorithm amounts to sampling from the IC empirical distribution. Let $F_{n_{\texttt{IC}}}$ denote the empirical version of $F_0$. Then we have $\widehat{\mathrm{ARL}}(\rho) = \mathscr{M}(F_{n_{\texttt{IC}}})$. It follows from Lemma \ref{lemma:gateaux} that  
\begin{align*}
    \left| \widehat{\mathrm{ARL}}(\rho) - \mathrm{ARL}(\rho) \right| &= \left| \mathscr{M}(F_{n_{\texttt{IC}}}) - \mathscr{M}(F_{0}) \right| \\
    &= \left| \int_0^1 \mathcal{D}_g \mathscr{M} \left(F_0 + \gamma(F_{n_{\texttt{IC}}} - F_0); F_{n_{\texttt{IC}}} - F_0 \right) \; d\gamma \right| \\
    &\leq \sum_{t = 1}^\infty t \left| \int_0^1 \eta_t \left(F_0 + \gamma(F_{n_{\texttt{IC}}} - F_0); F_{n_{\texttt{IC}}} - F_0 \right) \; d\gamma \right| \\
    &\leq \sum_{t = 1}^\infty t \int_0^1 C \cdot \frac{t \left\| F_0 + \gamma(F_{n_{\texttt{IC}}} - F_0) \right\|^{t - 1}_\infty \left\| F_{n_{\texttt{IC}}} - F_0 \right\|_\infty}{t^{3 + \xi}} \; d\gamma.
\end{align*}
Since $ \sup_{0 \leq \gamma \leq 1}\left\| (1 - \gamma) F_0 + \gamma F_{n_{\texttt{IC}}} \right\|_\infty \leq 1$, we have that 
\begin{align*}
    \left| \widehat{\mathrm{ARL}}(\rho) - \mathrm{ARL}(\rho) \right| = \mathcal{O}\left( \left\| F_{n_{\texttt{IC}}} - F_0  \right\|_\infty \right). 
\end{align*}
By the Dvoretzky-Kiefer-Wolfowitz Theorem, we have $\sqrt{n_{\texttt{IC}}}\left\| F_{n_{\texttt{IC}}} - F_0  \right\|_\infty = \mathcal{O}_p(1)$. This completes the proof. 
}

\section*{References Cited in the Proofs}

\begin{flushleft}
\par \noindent
Anthony, M., P. L. Bartlett. 1999. Neural network learning: Theoretical foundations, vol. 9. {\it Cambridge
University Press}, Cambridge.
\par \noindent
Bartlett, P. L., N. Harvey, C. Liaw, A. Mehrabian. 2019. Nearly-tight vc-dimension and pseudodimension bounds for piecewise linear neural networks. \textit{The Journal of Machine Learning Research} 20(1) 2285–2301.
\par \noindent
Bartlett, P. L., S. Mendelson. 2002. Rademacher and gaussian complexities: Risk bounds and structural
results. \textit{Journal of Machine Learning Research} 3(Nov) 463–482.
\par \noindent
Belloni, A., V. Chernozhukov. 2011. L1-penalized quantile regression in high-dimensional sparse models.
\textit{The Annals of Statistics} 39 82–130.
\par \noindent
Madrid Padilla, O. H., S. Chatterjee. 2022. Risk bounds for quantile trend filtering. \textit{Biometrika} 109(3) 751–768.
\par \noindent
Shen, Z., H. Yang, S. Zhang. 2019. Deep network approximation characterized by number of neurons.
arXiv preprint arXiv:1906.05497.
\par \noindent
Vershynin, R. 2018. High-dimensional probability: An introduction with applications in data science,
vol. 47. {\it Cambridge University Press}, Cambridge.
\end{flushleft}

\section{Supplemental Details for the Numerical Studies}\label{app:num}

\subsection{Details for Section \ref{sec:5.1}}\label{app:dagm}
First, we show a sample of DAGM images for a better understanding of this benchmark dataset. Figure \ref{fig:sample} shows the image sample. It can be seen from the figure that some surface defects are hard for human eyes to discern. Notably, all these background textures have certain random patterns involved (i.e., no gold standard), and some of them clearly do not meet the locality and stationarity assumptions of the  Markov Field approach (e.g., class 4, 7 and 10). 

\begin{figure}[ht!]
  \centering
\caption{Sample DAGM images: examples of IC and OC images.}
  \label{fig:sample} 
  \begin{tabular}{cccc}
    \includegraphics[width = 0.23\textwidth]{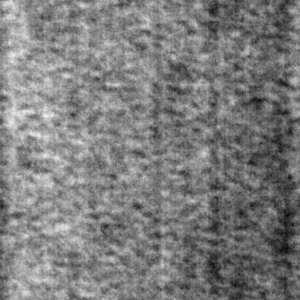} &
    \includegraphics[width = 0.23\textwidth]{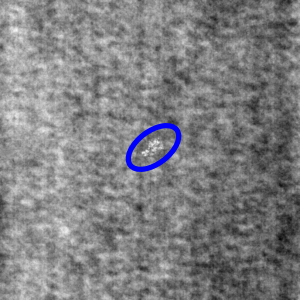} &
    \includegraphics[width = 0.23\textwidth]{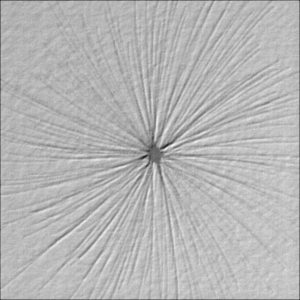} &
    \includegraphics[width = 0.23\textwidth]{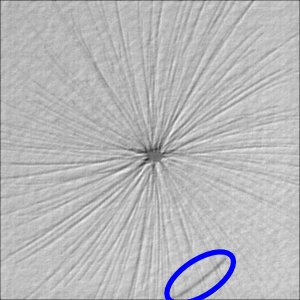} \\
    \includegraphics[width = 0.23\textwidth]{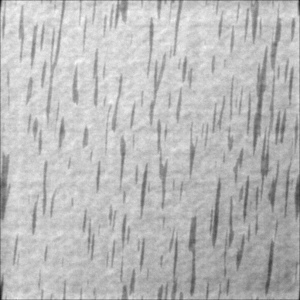} &
    \includegraphics[width = 0.23\textwidth]{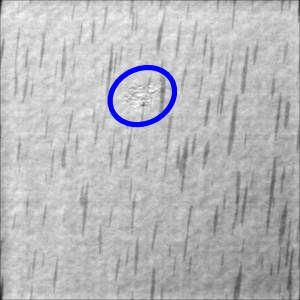} &
    \includegraphics[width = 0.23\textwidth]{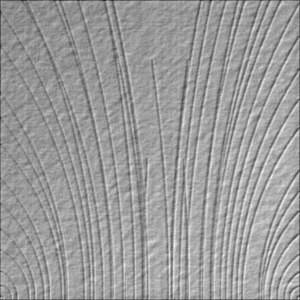} &
    \includegraphics[width = 0.23\textwidth]{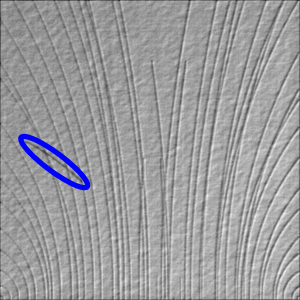} 
  \end{tabular}
 \vspace{8pt}
\fignote{Four groups of IC/OC images from class 3, 4, 7 and 10 in DAGM are shown Left-to-right and top-to-bottom.} 
\end{figure}

Second, to show more details about how we choose the stopping epoch for training MODERN-Net, we show the training progression for both networks $\mathscr{F}(\cdot; \mathcal{V})$ and $\mathscr{G}(\cdot; \mathcal{U})$ in Figure \ref{fig:epochs}. 
\begin{figure}[htp]
  \centering
    \caption{The training progression of $\mathscr{F}(\cdot; \mathcal{V})$ (left) and $\mathscr{G}(\cdot; \mathcal{U})$ (right) as the number of epochs increases.}
  \label{fig:epochs}
  \begin{tabular}{cc}
    \includegraphics[width = 0.48\textwidth, height = 0.28\textheight]{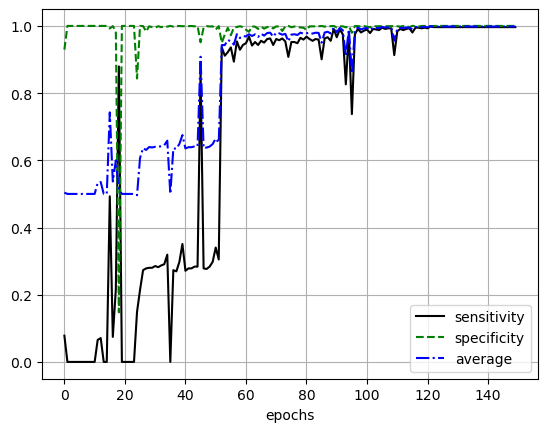} & \includegraphics[width = 0.48\textwidth, height = 0.28\textheight]{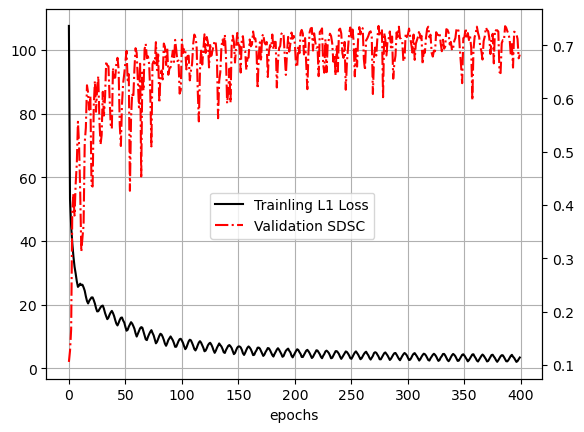} \\
  \end{tabular}
\end{figure}
For $\mathscr{F}(\cdot; \mathcal{V})$, the progression stabilizes from epoch 120, and thus we choose epoch 150 to stop training; on the other hand, for $\mathscr{G}(\cdot; \mathcal{U})$, the training progression stabilizes from epoch 250, and we stop at epoch 275.

\subsection{Details for Section \ref{sec:5.3}}\label{app:sim}

First, we provide some examples of the images simulated by the AR model described in Section \ref{sec:5.3}. Figure \ref{appfig:AR} shows an IC image along with two types of OC images to serve this purpose. It is clear that these images satisfy the Markov Field conditions required by the MF method. 
\begin{figure}[ht!]
\centering
\caption{Representative images in the AR simulated example.} \label{appfig:AR}
\begin{subfigure}[t]{.32\linewidth}
\centering
			 \includegraphics[width = \linewidth]{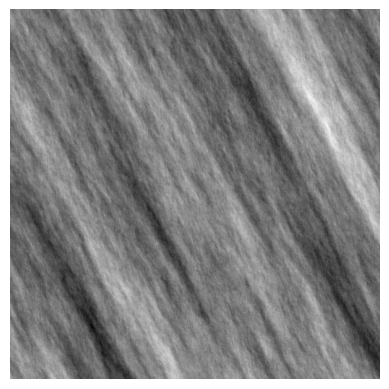}
             \caption{An IC image.}  
\end{subfigure}
~ 
\begin{subfigure}[t]{.32\linewidth}
\centering
			\includegraphics[width=\linewidth]{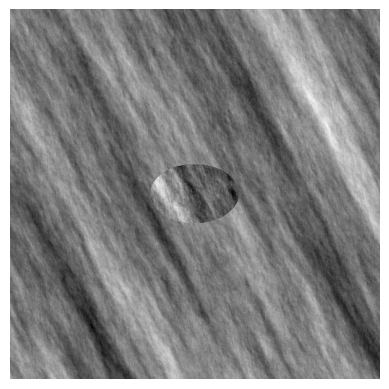} %
            \caption{A type-1 OC image.}  
\end{subfigure}
\begin{subfigure}[t]{.32\linewidth}
\centering
			\includegraphics[width=\linewidth]{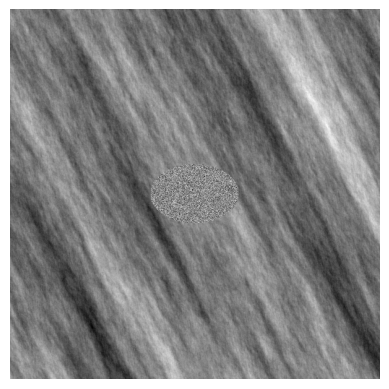} %
            \caption{A type-2 OC image.}  
\end{subfigure}
\end{figure}
Then, we assess the extent to which our MODERN-Net $\mathscr{F}(\cdot; \widehat{\mathcal{V}})$ is applicable to the AR images. After simulating 500 IC images and the two types of OC images (500 each type), we apply the applicability test to the images and obtain t-test statstics of values $415.13$ and $464.59$ for type-1 and type-2 faults, respectively, both rendering p-values effectively $0$. Note that, in addition to calculating the p-value, one can plot a histogram and observe whether $\mathscr{F}(\cdot; \widehat{\mathcal{V}})$ is able to separate the IC and OC images reasonably well. Such a visualization tool will greatly help with the assessment. Here, Figure \ref{fig:histAR} shows the histogram of the applicability test. It can be seen that the network is capable of distinguishing the IC and OC images quite well even though it has not been trained on any of these AR images. This is consistent with our p-value results. 
\begin{figure}[htp]
  \centering
  \caption{The histogram of the applicability test in which $\mathscr{F}(\cdot; \widehat{\mathcal{V}})$ is applied to 500 IC images and the two types of OC images (500 each type) generated by the AR model.}
  \label{fig:histAR}
  \includegraphics[width = 0.5\textwidth]{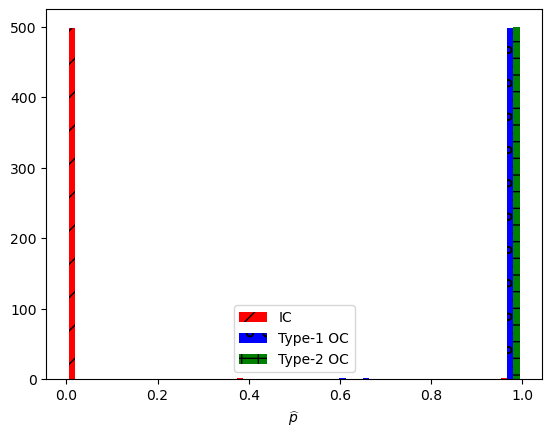}
\end{figure}
Lastly, Figure \ref{fig:AR} shows the fault isolation results corresponding to the case of median SDSC value in the repeated simulations. The MF approach did not detect any defective pixels in the case of type-1 fault, so its image is not shown in the figure. The true faulty area is marked by the solid blue ellipse. The MF approach detects faults in a pixel-by-pixel fashion; therefore, any pixels deemed abnormal are in black. It can be seen that our network $\mathscr{G}(\cdot; \widehat{\mathcal{U}})$ locates the faulty region accurately in both cases, whereas the MF approach can only detect type-2 fault (although also reasonably well).  
\begin{figure}[htp]
  \centering
  \caption{Fault isolation results give by MODERN-Diagnosis and the MF method in simulations.}
  \label{fig:AR}
\begin{subfigure}[t]{.32\linewidth}
\centering
			 \includegraphics[width = \textwidth]{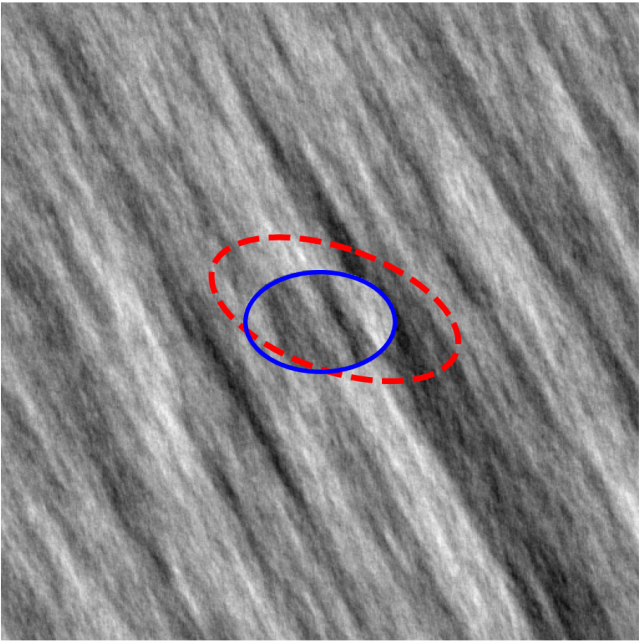}
             \caption{MODERN-Diagnosis result for type-1 fault.}  
\end{subfigure}
~ 
\begin{subfigure}[t]{.32\linewidth}
\centering
		 \includegraphics[width = \textwidth]{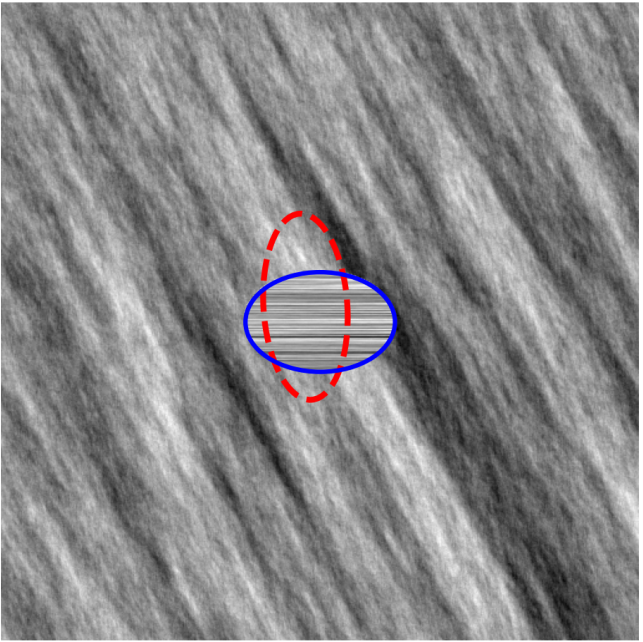}
            \caption{MODERN-Diagnosis result for type-2 fault.}  
\end{subfigure}
\begin{subfigure}[t]{.32\linewidth}
\centering
		\includegraphics[width = \textwidth, height = \textwidth]{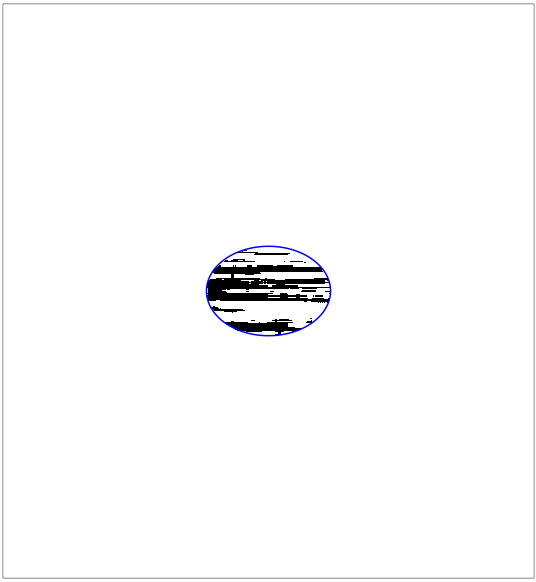}
            \caption{MF method result for type-2 fault.}  
\end{subfigure}
\vspace{8pt}
\fignote{The fault isolation results are corresponding to the case of median SDSC value in their repeated simulations.} 
\end{figure}

\subsection{Details for Section \ref{sec:5.4}}\label{app:steel}

First, to provide some examples of the image from the electric commutator manufacturing setting, we use Figure \ref{appfig:steel} to show one of the IC images and the (only) OC image, respectively. 
\begin{figure}[ht!]
\centering
\caption{Representative images in the electric commutator manufacturing.} \label{appfig:steel}
\begin{subfigure}[t]{.3\linewidth}
\centering
			 \includegraphics[width = \linewidth]{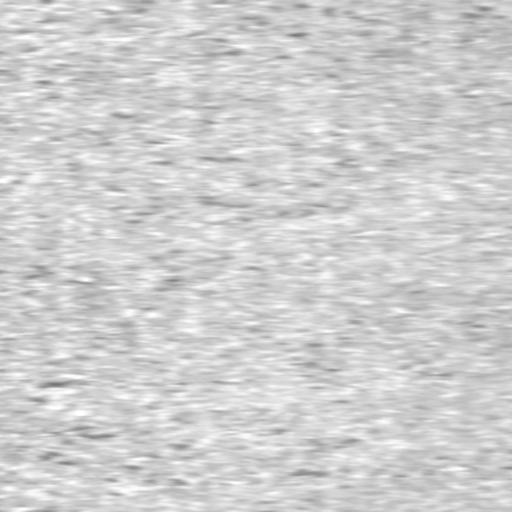}
             \caption{One of the 14 IC images.}  
\end{subfigure}
~ 
\begin{subfigure}[t]{.3\linewidth}
\centering
			\includegraphics[width=\linewidth]{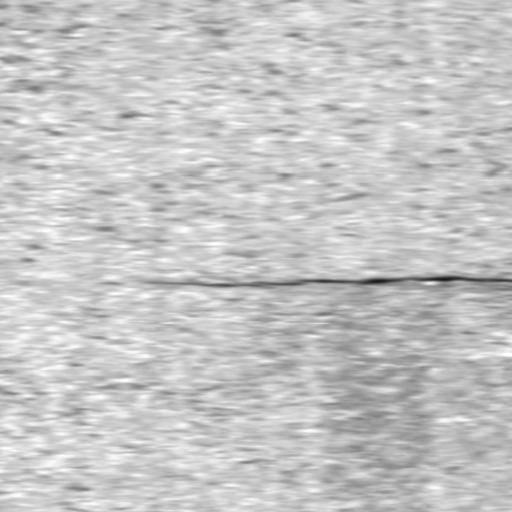} %
            \caption{The (only) OC image.}  
\end{subfigure}
\end{figure}
Second, we conduct the applicability test to ensure that the network trained on the benchmark dataset can be used in the considered manufacturing setting. Following Section \ref{sec:3.4}, we obtain the test statistic value of $22.67$, and the resulting p-value is smaller than $10^{-11}$. Moreover, as a direct visualization, Figure \ref{fig:steel_hist} clearly shows that the trained network can separate the OC image from other images quite well. 
\begin{figure}[htp]
  \centering
  \caption{The histogram of the applicability test for the commutator images}\label{fig:steel_hist}
  \includegraphics[width = 0.5\textwidth]{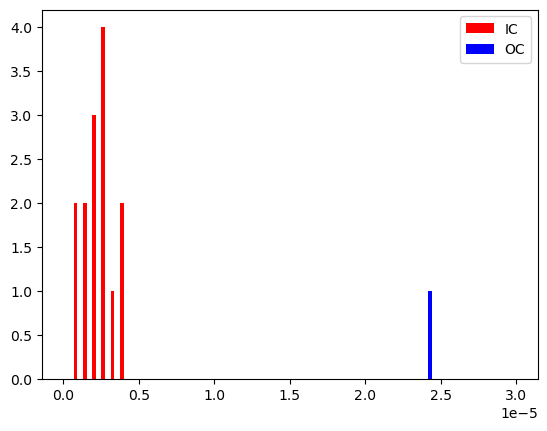}
\end{figure}

{\color{black}
\subsection{Comparison with AlexNet} \label{app:alex}
The specific ``'Inception ResNet'' architecture is highly regarded in the CNN literature and thus adopted by our MODERN framework. It is worth noting that other CNN architectures can be used in our framework without much difficulty. In this section, we benchmark our MODERN-Net performance against the well-known ``AlexNet'' \citep{krizhevsky2012imagenet}, the winner of the 2012 ImageNet competition. The architecture of ``AlexNet'' is detailed in Table \ref{tab:4}. It can be seen that ``AlexNet'' consists of the same CNN building blocks: convolution layers, ReLU activations, and pooling operations. Although less deep, ``AlexNet'' has 10 million more parameters than our MODERN-Net (33.4 million vs. 23.4 million).  We have modified the fully connected layers of the network so that it can work with $512\times 512$ images and output IC-OC probabilities. 

\begin{table}[ht]
  \centering
  \caption{The architecture of AlexNet.}
  \label{tab:4}
  \begin{tabular}{lcccc}
    \hline
    Layer         & Size                     & Kernel Size           & Stride & Pad  \\ \hline
    Input Image   & $512\times 512\times 3$  &                       &        &       \\
    Convolution 1 & $128\times 128\times 96$ & $11\times 11\times 3$ & 4      & 2     \\
    Batch Norm 1  & $128\times 128\times 96$ &                       &        &       \\
    Leaky ReLU 1  & $128\times 128\times 96$ &                       &        &        \\
    Max Pooling 1 & $63\times 63\times 96$   & $3\times 3$           & 2      & 0      \\
    Convolution 2 & $63\times 63\times 256$  & $5\times 5\times 96$  & 1      & 2     \\
    Batch Norm 2  & $63\times 63\times 256$  &                       &        &       \\
    Leaky ReLU 2  & $63\times 63\times 256$  &                       &        &        \\
    Max Pooling 2 & $31\times 31\times 256$  & $3\times 3$           & 2      & 0      \\
    Convolution 3 & $31\times 31\times 384$  & $3\times 3\times 256$ & 1      & 1     \\
    Batch Norm 3  & $31\times 31\times 384$  &                       &        &       \\
    Leaky ReLU 3  & $31\times 31\times 384$  &                       &        &        \\
    Convolution 4 & $31\times 31\times 384$  & $3\times 3\times 384$ & 1      & 1     \\
    Batch Norm 4  & $31\times 31\times 384$  &                       &        &       \\
    Leaky ReLU 4  & $31\times 31\times 384$  &                       &        &        \\
    Convolution 5 & $31\times 31\times 256$  & $3\times 3\times 384$ & 1      & 1     \\
    Batch Norm 5  & $31\times 31\times 256$  &                       &        &       \\
    Leaky ReLU 5  & $31\times 31\times 256$  &                       &        &        \\
    Max Pooling 5 & $15\times 15\times 256$  & $3\times 3$           & 2      & 0      \\
    Dropout  6    & 65536                    & \multicolumn{3}{c}{dropout rate = 0.5} \\
    Fully Connected 6          & 512                      &                       &        &        \\
    Leaky ReLU 6  & 512                      &                       &        &         \\
    Dropout  7    & 512                      & \multicolumn{3}{c}{dropout rate = 0.5} \\
    Fully Connected 7          & 256                      &                       &        &        \\
    Leaky ReLU 7  & 256                      &                       &        &         \\
    Fully Connected  8          & 2                        &                       &        &        \\
    \hline
  \end{tabular}
\end{table}

With the same training and validation DAGM datasets, we evaluate the classification performance of the two networks on the testing set. Regarding the training of ``AlexNet'', we employ the same stochastic gradient descent method with the triangular learning rate schedule as in Subsection \ref{sec:3.1} and augment the training set using the same image transformations $\mathcal{T}$ as in Subsection \ref{sec:5.1}. The comparison results are summarized in Table \ref{tab:5}. It can be seen that our MODERN-Net performed stably across the 10 image classes, whereas the performance of AlexNet varied substantially. In particular, AlexNet did not do well in classifying OC images. 

\begin{table}[h]
  \centering
  \caption{Performance comparison with AlexNet using DAGM.}
  \label{tab:5}
  \begin{tabular}{ccccc}
    \hline
          & \multicolumn{2}{c}{Specificity} & \multicolumn{2}{c}{Sensitivity} \\ 
    Class & MODERN         & AlexNet        & MODERN        &   AlexNet       \\ \hline
     1    & 100\%          & 100\%          & 100\%         & 100\%           \\ 
     2    & 100\%          & 100\%          & 100\%         & 100\%           \\ 
     3    & 100\%          & 100\%          & 100\%         & 100\%           \\ 
     4    & 100\%          & 98.63\%        & 100\%         & 42.31\%         \\ 
     5    & 100\%          & 100\%          & 100\%         & 90.48\%         \\ 
     6    & 100\%          & 100\%          & 100\%         & 90.48\%         \\ 
     7    & 100\%          & 100\%          & 100\%         & 100\%           \\ 
     8    & 100\%          & 100\%          & 97.06\%       & 97.06\%         \\ 
     9    & 100\%          & 100\%          & 100\%         & 100\%           \\ 
     10   & 100\%          & 100\%          & 100\%         & 98.11\%          \\ \hline
  \end{tabular}
\end{table}

}

\end{APPENDICES}

\end{document}